%% file: arxiv.tex
\pdfoutput=1 
\documentclass{sig-alternate-dan}
\usepackage{graphicx}
\usepackage{epsfig}
\usepackage{xspace}
\usepackage[boxed,noend,figure]{algorithm2e}
\usepackage[tight]{subfigure}

\usepackage{microtype}
\usepackage{color}
\usepackage{times}
\usepackage{varioref}

\usepackage{ifthen}
\newboolean{istechrpt}
\setboolean{istechrpt}{true}

\newboolean{includeblogs}
\setboolean{includeblogs}{false}

\newboolean{includecs}
\setboolean{includecs}{false}

\ifthenelse{\boolean{istechrpt}}{
}{ %

}
\numberwithin{equation}{section}

\input{macros}

\begin{document}

\newcommand{\instfont}[1]{\textrm{#1}}
\title{Online Distributed Sensor Selection}%
\author{Daniel Golovin\\Caltech %
\and Matthew Faulkner\\Caltech %
\and Andreas Krause\\Caltech %
}

\maketitle

\input{abstract}
\ifthenelse{\boolean{istechrpt}}{\vspace{-0.5em}}
{}
\vspace{-2mm}
\vspace{-2mm}

\input{introduction} %

\input{problemDef}
\input{centralized}
\input{distributed}

\input{lazyRenormalization}
\input{measurementDependentSampling}

\input{experiments}

\input{related}

\input{conclusions}
\ifthenelse{\boolean{istechrpt}}{\newpage}
{ }
\bibliographystyle{plain}
{\fontsize{9}{9}
\bibliography{ipsn-distonline}
}

\input{appendices}
\end{document}

%% file: macros.tex
\newcommand{\algofont}[1]{\textsc{#1}}
\newcommand{\algoname}{{\algofont{DOG}}\xspace}

\newcommand{\andreas}[1]{\textcolor{red}{\bf Andreas: #1}}

\newcommand{\OGunit}{$\algofont{OG}_{\algofont{unit}}$\xspace}
\newcommand{\cO}{\mathcal{O}}
\DeclareMathOperator{\EMSE}{EMSE}
\DeclareMathOperator{\MSE}{MSE}
\newcommand{\fEMSE}{f_{\EMSE}}

\newcommand{\WMR}{{\algofont{WMR}}\xspace}
\newcommand{\EXPthree}{{\algofont{EXP3}}\xspace}
\newcommand{\lamexp}{\ensuremath{\omega}} %
\newcommand{\numsensors}{\ensuremath{n}}
\newcommand{\probnull}{\ensuremath{\beta}}

\newcommand{\gout}{\ensuremath{G}}
\newcommand{\gpay}{\ensuremath{\psi}}
\newcommand{\paymaxact}[1]{\ensuremath{\max\paren{\payoff_{(A
        \setminus {#1})}}}}
\newcommand{\estimate}[1]{\bar{#1}}
\newcommand{\boost}{\ensuremath{\lambda}}
\newcommand{\thresh}{\ensuremath{\tau}}
\newcommand{\actcost}{\ensuremath{c}}
\newcommand{\oddog}{\algofont{OD-DOG}\xspace}
\newcommand{\dog}{\algofont{DOG}\xspace}
\newcommand{\lazydog}{\algofont{lazyDOG}\xspace}

\newcommand{\id}{\ensuremath{\operatorname{id}}}

\newcommand{\sensor}{\ensuremath{v}}
\newcommand{\sensors}{V}

\newcommand{\OPT}{\textsf{OPT}}
\newcommand{\NULL}{\textsf{null}}

\newcommand{\obj}{\ensuremath{f}} %

\newcommand{\K}{\ensuremath{k}}

\newcommand{\event}{\ensuremath{\mathcal{E}}}

\newcommand{\Alg}{\ensuremath{\mathcal{E}}}
\newcommand{\regret}{\ensuremath{R}}
\newcommand{\payoff}{\ensuremath{\pi}}

\newcommand{\psamp}{\ensuremath{\hat{p}}}
\newcommand{\Bernoulli}[1]{\ensuremath{\operatorname{Bernoulli}({#1})}}
\newcommand{\Binomial}[1]{\ensuremath{\operatorname{Binomial}({#1})}}
\newcommand{\Poisson}[1]{\ensuremath{\operatorname{Poisson}({#1})}}

\newcommand{\LimitProtocol}[0]{Poisson Multinomial Sampling
  Protocol\xspace}

\newcommand{\LimitProtocolShort}[0]{PMS Protocol\xspace}
\newcommand{\LimitProtocolShortns}[0]{PMS Protocol}

\newcommand{\mueager}{\ensuremath{\mu_{\text{eager}}}}
\newcommand{\mulazy}{\ensuremath{\mu_{\text{lazy}}}}
\newcommand{\ignore}[1]{}

\newcommand{\commentout}[1]{}

\newcommand{\cf}{\emph{cf.}}
\def \etal {\emph{et al.}\ }

\newcommand{\ith}{\ensuremath{i^{\mathrm{th}}}\xspace}

\newcommand{\citet}[1]{\cite{#1}}
\newcommand{\citep}[1]{\cite{#1}}

\newcommand{\figref}[1]{Fig.~\ref{#1}}
\newcommand{\vfigref}[1]{Fig.~\vref{#1}}

\newcommand{\secref}[1]{Sec.~\ref{#1}}

\newcommand{\propref}[1]{Prop.~\ref{#1}}

\newcommand{\denselist}{
    \itemsep -3pt\topsep-8pt\partopsep-8pt
}

\newcommand{\initOneLiners}{%
    \setlength{\itemsep}{0pt}
    \setlength{\parsep }{0pt}
    \setlength{\topsep }{0pt}
}
\newenvironment{OneLiners}[1][\ensuremath{\bullet}]
    {\begin{list}
        {#1}
        {\initOneLiners}}
    {\end{list}}

\newcommand{\nats}{\ensuremath{\mathbb{N}}}
\newcommand{\reals}{\ensuremath{\mathbb{R}}}
\newcommand{\Real}{\mathbb R}

\newcommand{\eps}{\varepsilon}

\newcommand{\class} [1] {\textrm{#1}} %
\renewcommand{\P} {\class{P}}
\newcommand{\NP} {\class{NP}}

\def \argmax {\mathop{\rm arg\,max}}

\newcommand{\paren} [1] {\ensuremath{ \left( {#1} \right) }}

\newcommand{\set}[1]{\ensuremath{\left\{#1\right\}}}

\newcommand{\prob}[1]{\ensuremath{\mathbf{Pr}\left[#1\right] }}

\newcommand{\E}[1]{\ensuremath{\mathbb{E}\left[#1\right] }}

\newcommand{\Econd}[2]{\ensuremath{\mathbb{E}\left[  {#1} \ \middle| \ {#2}\right] }}  %

\newcommand{\ceil}[1]{\ensuremath{\left\lceil#1\right\rceil}}

\newcommand{\tuple}[1]{\ensuremath{\langle #1 \rangle}}

\SetKw {KwEach} {\hspace{-1mm}each}
\SetKw {KwForEach} {for each}
\SetKw {KwFrom} {from}
\SetKw {KwSet} {set}
\SetKw {KwReturn} {return}
\SetKw {KwIf} {if}
\SetKw {KwThen} {then}
\renewcommand{\paragraph}[1]{\vspace{3mm}\noindent{\bf #1.} }

\newtheorem{lemma}{Lemma}
\newtheorem{theorem}{Theorem}
\newtheorem{cor}[theorem]{Corollary}
\newtheorem{prop}[theorem]{Proposition}

\newcommand{\cA}{{\mathcal{A}}}

\newcommand{\cB}{{\mathcal{B}}}
\newcommand{\cX}{{\mathcal{X}}}

\newcommand{\bx}{{\mathbf{x}}}

\def\more-auths{%
\end{tabular}
\begin{tabular}{c}}

%% file: abstract.tex
\begin{abstract}
A key problem in sensor networks is to decide which sensors to query
when, in order to obtain the most useful information (e.g., for performing
accurate prediction), subject to constraints (e.g., on power and
bandwidth). In many applications the utility function is not known a
priori, must be learned from data, and can even change over time.
Furthermore for large sensor networks solving a centralized
optimization problem to select sensors is not feasible, and thus we
seek a fully distributed solution.  In this paper, we present
\emph{Distributed Online Greedy} (\dog), an efficient, distributed
algorithm for repeatedly selecting sensors online, only receiving
feedback about the utility of the selected sensors.  We prove very
strong theoretical no-regret guarantees that apply whenever the
(unknown) utility function satisfies a natural diminishing returns
property called \emph{submodularity}.  Our algorithm has extremely low
communication requirements, and scales well to large sensor
deployments. We extend \dog to allow observation-dependent sensor
selection.  We empirically demonstrate the effectiveness of our
algorithm on several real-world sensing tasks.
\end{abstract}

%% file: introduction.tex
\section{Introduction}
\label{sec:intro}

A key challenge in deploying sensor networks for real-world
applications such as environmental monitoring \cite{krause07jmlr},
building automation \cite{singhvi05intelligent} and others is to
decide when to activate the sensors in order to obtain the most useful
information from the network (e.g., accurate predictions at unobserved
locations) and to minimize power consumption. This sensor selection
problem has received considerable attention
\cite{abrams04set,zhao02IDSQ,vldb04}, and algorithms with performance
guarantees have been developed \cite{abrams04set,krause05near}.
However, many of the existing approaches make simplifying
assumptions. Many approaches assume (1) that the sensors can perfectly
observe a particular sensing region, and nothing outside the region
\cite{abrams04set}.  This assumption does not allow us to model
settings where multiple noisy sensors can help each other obtain
better predictions.  There are also approaches that base their notion
of utility on more detailed models, such as improvement in prediction
accuracy w.r.t.~some statistical model \cite{vldb04} or detection
performance \cite{krause08efficient}.  However, most of these
approaches make two crucial assumptions: (2) The model, upon which the
optimization is based, is known in advance (e.g., based on domain
knowledge or data from a pilot deployment) and (3), a centralized
optimization selects the sensors (i.e., some centralized processor
selects the sensors which obtain highest utility w.r.t. the model).
We are not aware of any approach that simultaneously addresses the
three main challenges (1), (2) and (3) above and still provides
theoretical guarantees.

In this paper, we develop an efficient algorithm, called
\emph{Distributed Online Greedy} (\algoname), which addresses these
three central challenges.  Prior work \cite{krause07nearoptimal} has
shown that many sensing tasks satisfy an intuitive diminishing returns
property, submodularity, which states that activating a new sensor
helps more if few sensors have been activated so far, and less if many
sensors have already been activated.  Our algorithm applies to any
setting where the true objective is submodular \cite{nemhauser78},
thus capturing a variety of realistic sensor models. Secondly, our
algorithm does not require the model to be specified in advance: it
learns to optimize the objective function in an online manner.
Lastly, the algorithm is distributed; the sensors decide whether to
activate themselves based on local information.  We analyze our
algorithm in the no-regret model, proving convergence properties
similar to the best bounds for any centralized solution.

\paragraph{A bandit approach toward sensor selection} 
At the heart of our approach is a novel distributed algorithm for
multiarmed bandit (MAB) problems.  In the classical multiarmed bandit
\cite{robbins52some} setting, we picture a slot machine with multiple
arms, where each arm generates a random payoff with unknown mean.  Our
goal is to devise a strategy for pulling arms to maximize the total
reward accrued.  The difference between the optimal arm's payoff and
the obtained payoff is called the \emph{regret}. Known algorithms can achieve
average per-round regret of $\cO(\sqrt{n \log n} / \sqrt{T})$ where
$n$ is the number of arms, and $T$ the number of rounds (see e.g. the
survey of \cite{foster99}).  Suppose we would like to, at every time step,
select $k$ sensors.  The sensor selection problem can then be cast as
a multiarmed bandit problem, where there is one arm for each possible
set of $k$ sensors, and the payoff is the accrued utility for the
selected set.  Since the number of possible sets, and thus the number
of arms, is exponentially large,
the resulting regret bound is $\cO(n^{k/2}\sqrt{\log n} /
\sqrt{T})$, i.e., exponential in $k$.  However, when the utility
function is submodular, the payoffs of these arms are correlated.
Recent results \cite{streeter08} show that this correlation due to
submodularity can be exploited by reducing the $n^{k}$-armed bandit
problem to $k$ separate $n$-armed bandit problems, with only a bounded loss in
performance. 
\ignore{
 However, the existing no-regret algorithms for bandit
 optimization are centralized in nature: A centralized processor
 maintains a set of ``weights'' associated with each arm, which depend
 on the performance of the arms in past rounds.  Arms are then pulled
 at random by interpreting the weights as probabilities.The key
 question in distributed online submodular sensing is thus how to
 perform this sampling in a distributed fashion.  In this paper, we
 develop a scheme where each sensor maintains their own weights, and
 activates itself \emph{independently from all other sensors} purely
 depending on this weight. 
 }
Existing bandit algorithms, such as the widely used \EXPthree
algorithm \cite{auer02}, are centralized in nature. Consequently, the
key challenge in distributed online submodular sensing is how to
devise a distributed bandit algorithm. In Sec. \ref{sec:distributed}
and \ref{sec:lazy-renorm}, we develop a distributed variant of EXP3
using novel algorithms to sample from and update a probability
distribution in a distributed way. Roughly, we develop a scheme where
each sensor maintains its own weight, and activates itself
\emph{independently from all other sensors} purely depending on this
weight.

\paragraph{Observation specific selection} 
\ignore{
Another key question not addressed in centralized sensor selection is
the fact that sensors have more information about whether they should
become activated.
}
A shortcoming of centralized sensor selection is that the individual
sensors' current measurements are not considered in the selection
process.  In many applications, obtaining sensor measurements is 
less costly than transmitting the measurements across the network. For
example, cell phones used in participatory sensing
\cite{burke06participatory} can inexpensively obtain measurements on a
regular basis, but it is expensive to constantly communicate
measurements over the network. In
Sec. \ref{sec:measure-dependent-sampling}, we extend our distributed
selection algorithm to activate sensors depending on their
observations, and analyze the tradeoff between power consumption
and the utility obtained under observation specific activation.
\ignore{
when using cell phones in participatory sensing
\cite{burke06participatory}, it is inexpensive to obtain measurements
(such as temperature, GPS readings, etc.) on a regular basis, but
expensive to constantly communicate the measurements over the network.
In such a setting, one would like to activate sensors depending on
their observations.  We extend our algorithm to allow for observation
specific activation, and analyze the tradeoff between power
consumption and obtained utility.
}

\paragraph{Communication models}
We analyze our algorithms under two models of communication cost:
In the \emph{broadcast} model, each sensor can broadcast a
message to all other sensors at unit cost.
In the \emph{star network} model, messages can only be between a
  sensor and the base station, and each message has unit cost.
In \secref{sec:distributed} we formulate and analyze a distributed
algorithm for sensor selection under the simpler broadcast
model. Then, in \secref{sec:lazy-renorm} we show how the algorithm can
be extended to the star network model.

\paragraph{Our main contributions}
\begin{OneLiners}
\item Distributed \EXPthree, a novel distributed implementation of the
  classic multiarmed bandit algorithm.
\item Distributed Online Greedy (\dog) and \lazydog, novel algorithms for
  distributed online sensor selection, which apply to many settings,
  only requiring the utility function to be submodular.
\item \oddog, an extension of \dog to allow for
  observation-dependent selection.
\item We analyze our algorithm in the no-regret model and prove that
  it attains the optimal regret bounds attainable by any efficient
  centralized algorithm.
\item We evaluate our approach on several real-world sensing tasks including monitoring a 12,527 node network. 
\end{OneLiners}

\vspace{1em}
Finally, while we do not consider multi-hop or general
network topologies in this paper, 
we believe that the ideas behind our algorithms will likely prove
valuable for sensor selection in those models as well.

%% file: problemDef.tex
\section{The Sensor Selection Problem} \label{sec:problem}
We now formalize the sensor selection problem.  Suppose a network of
sensors has been deployed at a set of locations $\sensors$ with the
task of monitoring some phenomenon (e.g., temperature in a
building). 
Constraints on communication bandwidth or battery power typically
require us to select a subset $\cA$ of these sensors for activation,
according to some utility function.  The activated sensors then send
their data to a server (base station).  We first review the
traditional offline setting where the utility function is specified in
advance, illustrating how submodularity allows us to obtain provably near-optimal selections.
We then address the more challenging setting where the utility
function must be learned from data in an online manner.

\subsection{The Offline Sensor Selection Problem}
\label{sec:offline-problem}
A standard offline sensor selection algorithm chooses a set of sensors
that maximizes a known sensing quality objective function $f(\cA)$,
subject to some constraints,
e.g., on the number of activated sensors. One possible choice for the
sensing quality is based on prediction accuracy (we will discuss other
possible choices later on).  In many applications, measurements are
correlated across space, which allows us to make predictions at the
unobserved locations.  For example, prior work \cite{vldb04} has
considered the setting where a random variable $\cX_{s}$ is associated
with every location $s\in\sensors$, and a joint probability
distribution $P(\cX_{\sensors})$ models the correlation between sensor
values.  Here, $\cX_{\sensors}=[\cX_{1},\dots,\cX_{\numsensors}]$ is
the random vector over all measurements.  If some measurements
$\cX_{\cA}=\bx_{\cA}$ are obtained at a subset of locations, then the
conditional distribution
$P(\cX_{\sensors\setminus\cA}\mid\cX_{\cA}=\bx_{\cA})$ allows
predictions at the unobserved locations, e.g., by predicting
$\mathbb{E}[\cX_{\sensors\setminus\cA}\mid\cX_{\cA}=\bx_{\cA}]$.
Furthermore, this conditional distribution quantifies the
\emph{uncertainty} in the prediction: Intuitively, we would like to
select sensors that minimize the predictive uncertainty.  One way to
quantify the predictive uncertainty is the mean squared prediction
error,
$$\MSE(\cX_{\sensors\setminus\cA}\mid \bx_{\cA}) =
\frac{1}{\numsensors}\sum_{s\in\sensors\setminus\cA}
\mathbb{E}[(\cX_s-\mathbb{E}[\cX_{s}\mid\bx_{\cA}])^{2}\mid
\bx_{\cA}].$$ In general, the measurements $\bx_{\cA}$ that sensors
$\cA$ will make is not known in advance.  Thus, we can base our
optimization on the \emph{expected mean squared prediction error},
$$\EMSE(\cA) = \int dp(\bx_{\cA}) \MSE(\cX_{\sensors\setminus\cA}\mid \bx_{\cA}).$$
Equivalently, we can maximize the \emph{reduction} in mean squared prediction error,
$$\fEMSE(\cA) = \EMSE(\emptyset)-\EMSE(\cA).$$ By definition,
$\fEMSE(\emptyset)=0$, i.e., no sensors obtain no
utility. Furthermore, $\fEMSE$ is monotonic: if
$\cA\subseteq\cB\subseteq\sensors$, then $\fEMSE(\cA)\leq
\fEMSE(\cB)$, i.e., adding more sensors always helps.  That means,
$\fEMSE$ is maximized by the set of all sensors $\sensors$.  However,
in practice, we would like to only select a small set of, e.g., at
most $k$ sensors due to bandwidth and power constraints:
$$\cA^*=\argmax_{\cA}\fEMSE(\cA)\text{ s.t. }|\cA|\leq k.$$
Unfortunately, this optimization problem is NP-hard, so we cannot
expect to efficiently find the optimal solution.  Fortunately, it can
be shown \cite{das07} that in many settings\footnote{For Gaussian
models and conditional suppressorfreeness \cite{das07}}, the function
$\fEMSE$ satisfies an intuitive diminishing returns property called
submodularity.  A set function $f:2^{\sensors}\rightarrow\mathbb{R}$
is called \emph{submodular}
if, for all
$\cA\subseteq\cB\subseteq\sensors$ and $s\in\sensors\setminus \cB$ it
holds that $f(\cA\cup\{s\})-f(\cA)\geq f(\cB\cup\{s\})-f(\cB)$. Many
other natural objective functions for sensor selection satisfy
submodularity as well \cite{krause07nearoptimal}.  For example, the
\emph{sensing region} model where $\obj_{REG}(\cA)$ is the total area
covered by all sensors $\cA$ is submodular.  The \emph{detection}
model where $\obj_{DET}(\cA)$ counts the expected number of targets
detected by sensors $\cA$ is submodular as well.

A fundamental result of Nemhauser et al.~\cite{nemhauser78} is that
for monotone submodular functions, a simple greedy algorithm, which
starts with the empty set $\cA_{0}=\emptyset$ and iteratively adds the
element $$s_{k}=\argmax_{s\in\sensors\setminus \cA_{k-1}}
f(\cA_{k-1}\cup\{s\});\ \ \cA_{k}=\cA_{k-1}\cup\{s_{k}\}$$ which maximally improves the utility obtains a
near-optimal solution: For the set $\cA_{k}$ it holds that
$$f(\cA_{k})\geq (1-1/e)\max_{|\cA|\leq k}f(\cA),$$ i.e., the greedy
solution obtains at least a constant fraction of $(1-1/e)\approx 63\%$
of the optimal value.

\looseness -1 One fundamental problem with this offline approach is that it requires the function $f$ to be specified in advance, i.e., before running the greedy algorithm.  For the function $\fEMSE$, this means that the probabilistic model $P(\cX_{\sensors})$ needs to be known in advance.  While for some applications some prior data, e.g., from pilot deployments, may be accessible, very often no such prior data is available.  This leads to a ``chicken-and-egg'' problem, where sensors need to be activated to collect data in order to learn a model, but also the model is required to inform the sensor selection.  This is akin to the ``exploration--exploitation tradeoff'' in reinforcement learning \cite{auer02}, where an agent needs to decide whether to explore and gather information about effectiveness of an action, or to exploit, i.e., choose actions known to be effective. In the following, we devise an online monitoring scheme based on this analogy.

\subsection{The Online Sensor Selection Problem}  \label{sec:online-problem}
We now consider the more challenging problem where the objective
function is not specified in advance, and needs to be learned during
the monitoring task.  We assume that we intend to monitor the
environment for a number $T$ of time steps (rounds).  In each round
$t$, a set $S_{t}$ of sensors is selected, and these sensors transmit
their measurements to a server (base station).
The server then determines a sensing quality $\obj_t(S_t)$
quantifying the utility obtained from the resulting analysis. For
example, if our goal is spatial prediction, the server would build a
model based on the previously collected sensor data, pick a random
sensor $s$, make prediction for the variable $\cX_{s}$, and then
compare the prediction $\mu_{s}$ with the sensor reading $x_{s}$.  The
error $\obj_{t}=\sigma_{s}^{2}-(\mu_{s}-x_{s})^{2}$ is an unbiased
estimate of the reduction in EMSE.  In the following analysis, we will
only assume that the objective functions $\obj_{t}$ are bounded (w.l.o.g., take values in
$[0,1]$), monotone, and submodular, and that we have some way of
computing $\obj_{t}(S)$ for any subset of sensors $S$.
Our goal is to maximize the total reward obtained by the system 
over $T$ rounds, $\sum_{t=1}^T \obj_t(S_t)$.

We seek to develop a protocol for selecting the sets $S_{t}$ of
sensors at each round, such that after a small number of rounds the
average performance of our online algorithm converges to the same
performance of the offline strategy (that knows the objective
functions).  We thus compare our protocol against all strategies that
can select a fixed set of $\K$ sensors for use in all of the rounds;
the best such strategy obtains reward $\max_{S \subseteq \sensors: |S|
\le \K} \sum_{t=1}^T \obj_t(S)$.  The difference between this quantity
and what our protocol obtains is known as its \emph{regret}, and an
algorithm is said to be \emph{no-regret} if its average regret tends to zero (or less)\footnote{Formally, if $R_T$ is the
  total regret for the first $T$ rounds, no-regret means 
$\limsup_{T \to \infty} R_T/T \le 0$.} as $T \to \infty$.

When $\K=1$, our problem is
simply the well-studied \emph{multiarmed bandit} (MAB) problem, for
which many no-regret algorithms are known \cite{foster99}.  For
general $\K$, because the average of several submodular functions
remains submodular, we can apply the result of
Nemhauser~\etal~\cite{nemhauser78} (\cf, \secref{sec:offline-problem})
to prove that a simple greedy algorithm obtains a $(1-1/e)$
approximation to the optimal offline solution.
Feige~\cite{feige97} showed that this is optimal in the sense that
obtaining a \\$(1-1/e+\epsilon)$ approximation for any $\epsilon > 0$ is
\NP-hard.  These facts suggest that we cannot expect any efficient
online algorithm to converge to a solution better than\\
$(1-1/e)\max_{S \subseteq \sensors: |S| \le \K} \sum_{t=1}^T
\obj_t(S)$.  We therefore define the $(1-1/e)$-regret of a sequence of
(possibly random) sets $\set{S_t}_{t=1}^T$ as
\vspace{-0.5em}
\[
\regret_{T} := \paren{1-1/e} \cdot \max_{S \subseteq \sensors: |S| \le \K}
\sum_{t=1}^T \obj_t(S) \ - \ \sum_{t=1}^T \E{\obj_t(S_t)} 
\]
where the expectation is taken over the distribution for each $S_t$.
We say an online algorithm producing a sequence of sets has 
\emph{no-$(1-1/e)$-regret} if $\limsup_{T \to \infty}
\frac{\regret_{T}}{T} \le 0$. 

\ignore{
\begin{itemize}
\item Define sensor selection problem as online optimization problem
\item Objective functions are submodular (examples; mention Nemhauser result)
\item Define $(1-1/e)$-regret
\end{itemize}
} %

%% file: centralized.tex
\section{Centralized Algorithm for Online Sensor Selection}\label{sec:centralized}
Before developing the distributed algorithm for online sensor
selection, we will first review a centralized algorithm which is
guaranteed to achieve no $(1-1/e)$-regret. In
Sec. \ref{sec:distributed} we will show how this centralized algorithm
can be implemented efficiently in a distributed manner. This 
algorithm starts with the greedy algorithm for a known submodular
function mentioned in~\secref{sec:offline-problem}, and adapts it to the
online setting. 
Doing so requires an online algorithm for selecting a
\emph{single} sensor as a subroutine, and we review such an algorithm 
in~\secref{sec:centralized-sampling} before discussing the centralized
algorithm for selecting multiple sensors in~\secref{sec:centralized-multi}.

\ignore{
As we will
show, the key problem will be to develop an online algorithm for selecting a
\emph{single} sensor, which we consider in
\secref{sec:centralized-sampling}. 

In \secref{sec:centralized-multi}
we present our algorithm for selecting multiple sensors,
which relies on the single sensor case as a subroutine.
} %

\subsection{Centralized Online  Single Sensor Selection} \label{sec:centralized-sampling}
Let us first consider the case where $k=1$, i.e., we would like to
select one sensor at each round.  This simpler problem can be
interpreted as an instance of the multiarmed bandit problem (as
introduced in \secref{sec:online-problem}), where we have one arm for
each possible sensor.  In this case, the \EXPthree algorithm
\cite{auer02} is a centralized solution for no-regret single sensor
selection.
\EXPthree works as follows: It
is parameterized by a \emph{learning rate} $\eta$, and an
\emph{exploration probability} $\gamma$. It maintains a set of weights
$w_{s}$, one for each arm (sensor) $s$, initialized to
1. %
At every round $t$,
it will select each arm $s$ with probability
$$p_{s}=(1-\gamma)\frac{w_{s}}{\sum_{s'}
  w_{s'}}+\frac{\gamma}{\numsensors},$$ i.e., with probability $\gamma$
it explores, picking an arm uniformly at random, and with probability
$(1-\gamma)$ it exploits, picking an arm $s$ with probability
proportional to its weight $w_{s}$.  Once an arm $s$ has been
selected, a feedback $r=\obj_{t}(\{s\})$ is obtained, and the weight $w_{s}$ is
updated to
$$w_{s}\leftarrow w_{s}\exp(\eta r/p_{s}).$$ 
Auer et al.~\cite{auer02} showed that
with appropriately chosen learning rate $\eta$ and exploration
probability $\gamma$ it holds that the cumulative regret $R_{T}$ of
\EXPthree is $\cO(\sqrt{T \numsensors \ln \numsensors })$, i.e., the
average regret $R_{T}/T$ converges to zero.

\subsection{Centralized Selection of Multiple Sensors}\label{sec:centralized-multi}
In principle, we could interpret the sensor selection problem as a
$\binom{n}{k}$-armed bandit problem, and apply existing no-regret
algorithms such as \EXPthree. Unfortunately, this approach does not
scale, since the number of arms grows exponentially with $k$.
However, in contrast to the traditional multiarmed bandit problem,
where the arms are assumed to have independent payoffs, in the sensor
selection case, the utility function is submodular and thus the
payoffs are correlated across different sets.  Recently, Streeter and
Golovin showed how this submodularity can be exploited, and developed
a no-$(1-1/e)$-regret algorithm for online maximization of submodular
functions \cite{streeter08}.
The key idea behind their algorithm, $\text{OG}_{\text{unit}}$, is to
turn the offline greedy algorithm into an online algorithm by
replacing the greedy selection of the element $s_{k}$ that maximizes
the benefit $s_{k}=\argmax_{s} \obj(\{s_{1},...,s_{k-1}\}\cup\{s\})$
by a bandit algorithm. As shown in the pseudocode below, \OGunit
maintains $k$ bandit algorithms, one for each sensor to be selected.
At each round $t$, it selects $k$ sensors according to the choices of
the $k$ bandit algorithms $\Alg_{i}$ \footnote{Bandits with duplicate
choices are handled in Sec. 4.6.1 of \cite{streeter08}}.  Once the
elements have been selected, the $\ith$ bandit algorithm $\Alg_{i}$
receives as feedback the incremental benefit
$\obj_{t}(s_{1},\dots,s_{i})-\obj_{t}(s_{1},\dots,s_{i-1})$, i.e., how
much additional utility is obtained by adding sensor $s_{i}$ to the
set of already selected sensors.  Below we define $[m] := \set{1, 2,
  \ldots, m}$.

\begin{center}
\begin{tabular}{|p{8cm}|}
\hline
\textbf{Algorithm \OGunit from~\protect{\cite{streeter08}}:}\\
\hspace{0cm} Initialize $\K$ multiarmed bandit algorithms  $\Alg_1, \Alg_2, \ldots, \Alg_\K$,\\ 
\hspace{0cm} each with action set $\sensors$.\\
\hspace{0cm} For each round $t \in [T]$\\
\hspace{0.3cm} For each stage $i \in [\K]$ in parallel\\
\hspace{0.6cm}      $\Alg_i$ selects an action $\sensor^t_i$\\
\hspace{0.3cm} For each $i \in [\K]$ in parallel\\
\hspace{0.6cm}      feedback $\obj_t(\set{\sensor^t_j : j
  \le i }) - \obj_t(\set{\sensor^t_j : j < i })$ to $\Alg_i$.\\
\hspace{0.3cm} Output $S_t = \set{a^t_1, a^t_2, \ldots, a^t_\K}$.\\
\hline
\end{tabular}
\end{center}

In~\cite{streeter07tr} it is shown that \OGunit has a
$\paren{1-\frac{1}{e}}$-regret bound of $\cO(\K R)$ in this feedback
model assuming each $\Alg_i$ has expected regret at most $R$. Thus, when using \EXPthree as a subroutine, \OGunit has no-$(1-1/e)$-regret.

Unfortunately, \EXPthree (and in fact all MAB algorithms with
no-regret guarantees for non-stochastic reward functions) require
sampling from some distribution with weights associated with the
sensors.
If $\numsensors$ is small, we could simply store these weights on the
server, and run the bandit algorithms $\Alg_i$ there.  However, this
solution does not scale to large numbers of sensors.  
Thus the key
problem for online sensor selection is to develop a multiarmed bandit
algorithm which implements \emph{distributed sampling} across the
network, with minimal overhead of communication.  In addition, the
algorithm needs to be able to maintain the distributions (the weights)
associated with each $\Alg_{i}$ in a distributed fashion.
\ignore{  
In the following, we will develop \algoname, an efficient distributed
algorithm that address these key challenges.
} %

%% file: distributed.tex
\section{Distributed Algorithm for \\Online Sensor Selection}\label{sec:distributed}

We will now develop \algoname, an efficient algorithm for distributed
online sensor selection.  For now we make the following assumptions:

\begin{enumerate}\denselist
\item Each sensor $\sensor \in \sensors$ is able to compute its
contribution to the utility $\obj_{t}(S\cup \{\sensor\})-\obj_{t}(S)$,
where $S$ are a subset of sensors that have already been selected.
\item Each sensor can broadcast  to all other sensors.
\item The sensors have calibrated clocks and unique, linearly ordered identifiers.
\end{enumerate}

These assumptions are reasonable in many applications: (1) In target
detection, for example, the objective function $\obj_{t}(S)$ counts
the number of targets detected by the sensors $S$.  Once previously
selected sensors have broadcasted which targets they detected, the new
sensor $s$ can determine how many additional targets have been
detected.  Similarly, in statistical estimation, one sensor (or a
small number of sensors) randomly activates each round and broadcasts
its value. After sensors $S$ have been selected and announced their
measurements, the new sensor $s$ can then compute the improvement in
prediction accuracy over the previously collected data. (2) The
assumption that broadcasts are possible may be realistic for dense
deployments and fairly long range transmissions.  In
\secref{sec:lazy-renorm} we will show how assumptions (1) and (2) can
be relaxed.

As we have seen in \secref{sec:centralized}, the key insight in
developing a centralized algorithm for online selection is to replace
the greedy selection of the sensor which maximally improves the total
utility over the set of previously selected sensors by a bandit
algorithm.  Thus, a natural approach for developing a distributed
algorithm for sensor selection is to first consider the single sensor
case.

\subsection{Distributed Selection of a Single Sensor}\label{sec:sampling}

\ignore{ If we wish to implement a distributed variant of a bandit
algorithm such as \EXPthree, there are two questions we must answer,
namely, how do we sample arms and how do we update the distribution on
arms?  We will now develop a distributed variant of \EXPthree that
solves these challenges.  } %

\commentout{
\andreas{need to move:
\begin{theorem} \label{thm:distributed-single-sensor-selection}
For every $\alpha \in [1, \infty)$ there is a protocol for the
  multiarmed bandit problem such that the worst case expected regret
  is
$\cO(\sqrt{\OPT \numsensors \log \numsensors} + \OPT e^{-\alpha})$, where
  $\OPT$ is the total reward of the best action, and in each round at most 
$\alpha\paren{1 + (e-1)/\alpha}$ sensors communicate with the server and at most 
$2\alpha\paren{1 + (e-1)/\alpha}$ messages are exchanged in
expectation.
Moreover, there is a protocol with worst case expected regret
$\cO(\sqrt{\OPT \numsensors \log \numsensors})$ and in each round 
only $\cO(\log \numsensors)$ sensors  communicate with the server and 
$\cO(\log \numsensors)$ messages are exchanged with high probability.
\end{theorem}
}
} %

\ignore{
The key challenge in developing a distributed version of \EXPthree
when sensors can broadcast to all other sensors is to find a way to
sample exactly one element from a probability distribution $p$ over
sensors in a distributed manner.  We measure the cost of the sampling
procedure in terms of the number of broadcast messages.
}

\looseness -1 The key challenge in developing a distributed version of \EXPthree is
to find a way to sample exactly one element from a probability
distribution $p$ over sensors in a distributed manner. This problem is
distinct from randomized leader election \cite{Nakano02}, where the
objective is to select exactly one element but the element need not be
drawn from a specified distribution.  We note that under the multi-hop
communication model, sampling one element from the uniform
distribution given a rooted spanning tree can be done via a simple
random walk~\cite{Kuhn08}, but that under the broadcast and star network models
this approach degenerates to centralized sampling.  
Our algorithm, in contrast, samples from an arbitrary
distribution by allowing sensors to individually decide to activate.
Our bottom-up approach also has two other advantages:
($1$) it is amenable to modification of the activation probabilities based on
local observations, as we discuss in~\secref{sec:measure-dependent-sampling}, and ($2$)
since it does not rely on any global state of the network such as a
spanning tree, it can gracefully cope with significant edge or node
failures.

\ignore{
Our proposed algorithm allows sensors to individually decide to
activate, with the result that the selected sensor is drawn from the
desired distribution.

Our proposed is suited for the broadcast and star network models.

Rooted spanning tree algorithms reduce to centralized sampling in
these models.

Kuhn has proposed a rooted spanning tree algorithm to select an
element of a network uniformly at random.

Kuhn's method assumes a multi-hop network, with rooted spanning tree.

We measure the cost of the sampling procedure in terms of the number
of broadcast messages.
}

\paragraph{A naive distributed sampling scheme}
A naive distributed algorithm would be to let each sensor keep track
of all activation probabilities $p$.  Then, one sensor (e.g., with the
lowest identifier) would broadcast a single random number $u$
uniformly distributed in $[0,1]$, and the sensor $\sensor$ for which
$\sum_{i=1}^{\sensor-1} p_{i} \le u < \sum_{i=1}^{\sensor} p_{i}$
would activate.  However, for large sensor network deployments, this
algorithm would require each sensor to store a large amount of global
information (all activation probabilities $p$).  Instead, each sensor
$\sensor$ could store only their own probability mass $p_{\sensor}$;
the sensors would then, in order of their identifiers, broadcast their
probabilities $p_{\sensor}$, and stop once the sum of the
probabilities exceeds $u$.  This approach only requires a constant
amount of local information, but requires an impractical $\Theta(n)$
messages to be sent, and sent sequentially over $\Theta(n)$ time
steps.

\paragraph{Distributed multinomial sampling}
In this section we present a protocol that requires only $\cO(1)$
messages in expectation, and only a constant amount of local information.

For a sampling procedure with input distribution $p$, we let $\psamp$
denote the resulting distribution, where in all cases at most one
sensor is selected, and nothing is selected with probability
$1-\sum_{\sensor} \psamp_{\sensor}$.  A simple approach towards
distributed sampling would be to activate each sensor
$\sensor\in\sensors$ \emph{independently} from each other with
probability $p_{\sensor}$. While in expectation, exactly one sensor is
activated, with probability $\prod_{\sensor} (1-p_{\sensor})>0$ no
sensor is activated; also since sensors are activated independently,
there is a nonzero probability that more than one sensor is activated.
Using a synchronized clock, the sensors could determine if no sensor
is activated.  In this case, they could simply repeat the selection
procedure until at least one sensor is activated.  
One naive approach would be to repeat the selection
procedure until exactly one sensor is activated.  
However with two sensors and $p_1 = \eps, p_2 = 1 - \eps$ this
algorithm yields $\psamp_1 = \eps^2 / (1 - 2\eps + 2 \eps^2) =
\cO(\eps^2)$, so the first sensor is severely underrepresented.
Another simple protocol would be to select exactly one sensor uniformly at
random from the set of activated sensors, which can be implemented using few messages.

\begin{center}
\begin{tabular}{|p{7.6cm}|}
\hline
\textbf{The Simple Protocol:}\\
\hspace{0cm} For each sensor $\sensor$ in parallel\\
\hspace{0.5cm} Sample $X_{\sensor} \sim \Bernoulli{p_{\sensor}}$.\\
\hspace{0.5cm} If $(X_{\sensor} = 1)$, $X_{\sensor}$ activates.\\ 
\hspace{0cm} All active sensors $S$ coordinate to select a single
sensor uniformly at random from $S$, e.g., by electing the minimum ID
sensor in $S$ to do the sampling.\\
\hline
\end{tabular}
\end{center}

\noindent
It is not hard to show that with this protocol, for all sensors $\sensor$, 
\[
\psamp_{\sensor} = p_{\sensor}\cdot \Econd{\frac{1}{|S|}}{\sensor \in S} \ge p_{\sensor}/ \Econd{|S|}{\sensor \in S} \ge
p_{\sensor}/2 
\] 
by appealing to Jensen's inequality.  Since $\psamp_{\sensor} \le
p_{\sensor}$, we find that this simple protocol maintains a ratio
$r_{\sensor} := \psamp_{\sensor} / p_{\sensor} \in [\frac12 , 1]$.
Unfortunately, this analysis is tight, as can be seen from the example
with two sensors and $p_1 = \eps, p_2 = 1 - \eps$.

To improve upon the simple protocol, first consider running it on an
example with $p_1 = p_2 = \cdots = p_n = 1/n$.  Since the protocol
behaves exactly the same under permutations of sensor labels, by
symmetry we have $\psamp_{1} = \psamp_{2} = \cdots = \psamp_{n}$, and
thus $r_i = r_j$ for all $i,j$.  Now consider an input distribution
$p$ where there exists integers $N$ and $k_1, k_2, \ldots, k_n$ such
that $p_{\sensor} = k_{\sensor}/N$ for all $\sensor$.  Replace each
$\sensor$ with $k_{\sensor}$ fictitious sensors, each with probability
mass $1/N$, and each with a label indicating $\sensor$.  Run the
simple protocol with the fictitious sensors, selecting a fictitious
sensor $\sensor'$, and then actually select the sensor indicated by
the label of $\sensor'$.  By symmetry this process selects each
fictitious sensor with probability $(1-\probnull)/N$, where
$\probnull$ is the probability that nothing at all is selected, and
thus the process selects sensor $\sensor$ with probability
$k_{\sensor}(1-\probnull)/N = (1-\probnull)p_{\sensor}$ (since at most
one fictitious sensor is ever selected).

We may thus consider the following improved protocol which
incorporates the above idea, simulating this modification to the
protocol exactly when $p_{\sensor} = k_{\sensor}/N$ for all $\sensor$.

\begin{center}
\begin{tabular}{|p{7.6cm}|}
\hline
\textbf{The Improved Protocol($N$):}\\
\hspace{0cm} For each sensor $\sensor$ in parallel\\
\hspace{0.5cm} Sample $X_{\sensor} \sim \Binomial{\ceil{N \cdot
    p_{\sensor}}, 1/N}$.\\
\hspace{0.5cm} If $(X_{\sensor} \ge 1)$, then activate sensor $\sensor$.\\ 
\hspace{0cm} From the active sensors $S$, select sensor $\sensor$ with
probability $X_{\sensor}/\sum_{\sensor'\in S} X_{\sensor'}$.\\
\hline
\end{tabular}
\end{center}

This protocol ensures the ratios $r_{\sensor} := \psamp_{\sensor} /
p_{\sensor}$ are the same for all sensors, provided each $p_{\sensor}$
is a multiple of $1/N$.  Assuming the probabilities are rational,
there will be a sufficiently large $N$ to satisfy this condition.
To reduce $\probnull := \prob{S = \emptyset}$ in the simple protocol, we
may sample each $X_{\sensor}$ from $\Bernoulli{\alpha \cdot p_{\sensor}}$ for any
$\alpha \in [1, n]$.  The symmetry argument remains unchanged.
This in turn suggests sampling $X_{\sensor}$ from 
$\Binomial{\ceil{N \cdot p_{\sensor}}, \alpha/N}$ in the improved
protocol.  Taking the limit as $N \to \infty$, the binomial
distribution becomes Poisson, and we obtain the desired protocol.

\begin{center}
\begin{tabular}{|p{8.2cm}|}
\hline
\textbf{The Poisson Multinomial Sampling (PMS)
  Protocol($\alpha$):}\\
\hspace{0cm} Same as the improved protocol, except each \\ 
\hspace{0cm} sensor $\sensor$ samples $X_{\sensor} \sim \Poisson{\alpha p_{\sensor} }$\\
\ignore{
\hspace{0cm} \emph{Sensors:} \\
\hspace{0cm} For each sensor $\sensor$ in parallel\\
\hspace{0.5cm} Sample $X_{\sensor} \sim \Poisson{\alpha p_{\sensor} }$.\\
\hspace{0.5cm} If $(X_{\sensor} \ge 1)$, send $X_{\sensor}$ to the server.\\ 
\hspace{0cm} \emph{Server:}\\
\hspace{0cm} Receive messages from a set $S$ of sensors. \\
\hspace{0cm} If $S = \emptyset$, select nothing.\\
\hspace{0cm} Else, select $\sensor$ with probability
$X_{\sensor}/\sum_{\sensor' \in S} X_{\sensor'}$.\\
}
\hline
\end{tabular}
\end{center}
\noindent
Straight-forward calculation shows that 
\[
\prob{S = \emptyset}  = \prod_{\sensor} \exp\set{ - \alpha \cdot
  p_{\sensor}}  = \exp \bigl\{ - \sum_{\sensor} \alpha \cdot p_{\sensor}\bigr\}  = e^{-\alpha}  
\]
\ignore{
\begin{align*}
\prob{S = \emptyset}  
&= \prod_{\sensor} \exp\set{ - \alpha \cdot p_{\sensor}} \\
& = \exp \bigl\{ - \sum_{\sensor} \alpha \cdot p_{\sensor}\bigr\}  = \exp \set{-\alpha}   
\end{align*}
} %
Let $C$ be the number of messages.  Then 
\[ \E{C} =  \sum_{\sensor} \prob{X_{\sensor} \ge 1} =  \sum_{\sensor}
(1-e^{-\alpha p_{\sensor}}) \le \sum_{\sensor} \alpha p_{\sensor}   =
\alpha \]
\ignore{
\begin{align*}
\E{C} & =  \sum_{\sensor} \prob{X_{\sensor} \ge 1} \\
 & = \sum_{\sensor} (1-\exp\set{-\alpha p_{\sensor}})\\
 & \le  \sum_{\sensor} \alpha p_{\sensor} = \alpha\mbox{.}
\end{align*}
Here we have used linearity of expectation in the first line, and $1+x
\le e^x$ for all $x \in \Real$ in the third.
} %
Here we have used linearity of expectation, and $1+x
\le e^x$ for all $x \in \Real$.
In summary, we have the following result about our protocol:
\begin{prop} \label{thm:sampling}
Fix any fixed  $p$ and $\alpha>0$.  The \LimitProtocolShort always selects at most
one sensor, ensures   
$$\forall \sensor:\ \prob{\sensor \text{ selected}} = (1 - e^{-\alpha})p_{\sensor}$$
and requires no more than $\alpha$ messages in expectation. 
\end{prop}

In order to ensure that exactly one sensor is selected, whenever $S =
\emptyset$ we can simply rerun the protocol with fresh random seeds as
many times as needed until $S$ is non-empty.  Using $\alpha = 1$, 
this modification will require only $\cO(1)$ messages in expectation and
at most $\cO(\log n)$ messages with high probability in the broadcast model.
We can combine this protocol with \EXPthree to get the following
result.

\begin{theorem} \label{thm:distributed-single-sensor-selection}
In the broadcast model, running \EXPthree using the \LimitProtocolShort with
$\alpha = 1$, and rerunning the protocol whenever nothing is selected,
yields exactly the same regret bound as standard \EXPthree, and
 in each round at most $e/(e-1)+2 \approx 3.582$ messages
are broadcast in expectation.
\end{theorem}

The regret bound for \EXPthree is $\cO(\sqrt{\OPT \numsensors \log \numsensors})$, where
$\OPT$ is the total reward of the best action.  
Our variant simulates \EXPthree, and thus has identical regret.
Proofs of our theoretical results can be found in\ifthenelse{\boolean{istechrpt}}{
\noindent
the
Appendix. }{ %
the full version of this paper
\cite{golovin09distributed}.  }

\paragraph{Remark} Running our variant of \EXPthree requires that
each sensor know the number of sensors, $n$, in order to compute its
activation probability.  If each sensor $\sensor$ has only a reasonable
estimate of $n_{\sensor}$ of $n$, however, our algorithm still
performs well.  
For example, it is possible to prove that if all of the sensors have
the same estimate $n_{\sensor} = cn$ for some constant $c > 0$, 
then the upper bound on expected 
regret, $R(c)$, grows as $R(c) \approx  R(1)\cdot \max \set{c,
  1/c}$.
The expected number of activations in this
case increases by at most $\paren{\frac{1}{c} - 1}\gamma$.
In general underestimating $n$ leads to more activations, and
underestimating or overestimating $n$ can lead to more regret.
This graceful degradation of performance with respect to the error
in estimating $n$ holds for all of our algorithms.

\subsection{The Distributed Online Greedy Algorithm}
We now use our single sensor selection algorithm to develop our main
algorithm, the Distributed Online Greedy algorithm (\algoname).  It is
based on the distributed implementation of \EXPthree using the
\LimitProtocolShort. 
Suppose we would like to select $k$ sensors at each round
$t$.  Each sensor $\sensor$ maintains $k$ weights
$w_{\sensor,1},\dots,w_{\sensor,k}$ and normalizing constants
$Z_{\sensor,1},\dots,Z_{\sensor,k}$. The algorithm proceeds in $k$
stages, synchronized using the common clock. In stage $i$, a single
sensor is selected using the \LimitProtocolShort applied to the
distribution $(1-\gamma)w_{\sensor,i}/Z_{\sensor,i} + \gamma/n$.
Suppose sensors $S=\{\sensor_{1},\dots,\sensor_{i-1}\}$ have been
selected in stages $1$ through $i-1$.  The sensor $\sensor$ selected
at stage $i$ then computes its local rewards $\payoff_{\sensor,i}$
using the utility function
$\obj_{t}(S\cup\{\sensor_{i}\})-\obj_{t}(S)$. It then computes its new
weight
$$w'_{\sensor,i}=w_{\sensor,i}\exp(\eta \payoff_{\sensor,i} /
p_{\sensor,i}),$$ and broadcasts the difference between its new and
old weights $\Delta_{\sensor,i}=w'_{\sensor,i}-w_{\sensor,i}$.  All
sensors then update their $\ith$ normalizers using
$Z_{\sensor,i}\leftarrow Z_{\sensor,i}+\Delta_{\sensor,i}$.
\vfigref{fig:dog-code} presents the pseudo-code of the \algoname
algorithm.  Thus given Theorem~$12$ of~\cite{streeter07tr} we have the
following result about the \algoname algorithm:
\begin{theorem}\label{thm:dog-performance}
The \algoname algorithm selects, at each round $t$ a set
$S_{t}\subseteq\sensors$ of $k$ sensors such that
\[
\frac{1}{T}\mathbb{E}\left[\sum_{t=1}^{T}f_{t}(S_{t})\right]\geq
\frac{1-\frac{1}{e}}{T}\max_{|S|\leq k} \sum_{t=1}^{T} f_{t}(S) -
O\left(k\sqrt{\frac{n\log n}{T}}\right).
\]
In expectation, only $\cO(k)$ messages are exchanged each round. 
\end{theorem}

\input{DOGpseudocode}

%% file: DOGpseudocode.tex
\begin {algorithm*}%
\Titleofalgo { Distributed Online Greedy (\dog) \ \ (described in the broadcast model) }
\KwIn { $\K \in \nats$, a set $\sensors$, and $\alpha, \gamma, \eta
  \in \reals_{>0}$.  Reasonable defaults are any $\alpha \in [1, \ln
    |\sensors|]$, and $\gamma = \eta$ $=$ $\min\paren{1, \paren{|\sensors|
      \ln |\sensors|/g}^{1/2}}$, where $g$ is a guess for the maximum cumulative
reward of any single sensor~\cite{auer02}. }
Initialize $w_{\sensor, i} \gets 1$ and $Z_{\sensor,i} \gets |\sensors|$ for all
$\sensor \in \sensors$, $i \in [\K]$.
Let  $\rho(x, y) := (1-\gamma)\frac{x}{y} +
\frac{\gamma}{|\sensors|}$.\\
\For{ \KwEach \emph{round} $t = 1, 2, 3, \ldots $} {
  Initialize $S_{\sensor,t} \gets \emptyset$ for each $\sensor$ in parallel.\\
  \For{ \KwEach \emph{stage} $i \in [\K]$} {
    \For{ \KwEach \emph{sensor} $\sensor \in \sensors$ \emph{in parallel}} { 
      \Repeat {\emph{$\sensor$ receives a message of type }
        $\tuple{\emph{select } \id}$ } {
         Sample $X_{\sensor} \sim \Poisson{\alpha \cdot \rho(w_{\sensor,i}, Z_{\sensor,i})} $.\\
         \If{$(X_{\sensor} \ge 1)$} { 
           Broadcast $\tuple{\text{sampled } X_{\sensor}, \id(\sensor)}$; Receive messages from sensors $S$. (Include $\sensor \in S$
           for convenience).\\
           \If{$\id(\sensor) = \min_{\sensor' \in S} \id(\sensor')$}{
             Select exactly one element $v_{it}$ from $S$
             such that each $\sensor'$ is selected with probability
             $X_{\sensor'}/\sum_{u \in S} X_{u}$.\\              
             Broadcast $\tuple{\text{select } \id(\sensor_{it})}$.\\
           }
           Receive message $\tuple{\text{select } \id(\sensor_{it})}$.\\
           \If{$\id(\sensor) = \id(\sensor_{it})$}{
             Observe $\obj_t(S_{\sensor,t} + \sensor)$; $\payoff \gets
             \obj_t(S_{\sensor,t} + \sensor) -  \obj_t(S_{\sensor,t})$; $\Delta_{\sensor} \gets  w_{\sensor,i} (\exp\set{\eta \cdot
               \payoff/ \rho(w_{\sensor,i}, Z_{\sensor,i})} - 1)$;
             $Z_{\sensor,i} \gets Z_{\sensor,i} + \Delta_{\sensor}$;
             $w_{\sensor} \gets w_{\sensor} + \Delta_{\sensor}$;
             Broadcast $\tuple{\text{weight update }\Delta_{\sensor}, \id(\sensor)}$.\\
           }
         }
         \lIf{\emph{receive message $\tuple{\text{weight update }\Delta, \id(\sensor_{it})}$}}{
           $S_{\sensor,t} \gets S_{\sensor,t} \cup \set{\sensor_{it}}$;
           $Z_{\sensor,i} \gets Z_{\sensor,i} + \Delta $;\\
         }
        } %
    }
  }
}
\vspace{2mm}
 \KwOut{At the end of each round $t$ each sensor has an identical local copy
  $S_{\sensor,t}$ of the selected set $S_{t}$.}
\caption{The Distributed Online Greedy Algorithm}
\label{fig:dog-code}
\vspace{-5mm}
\end {algorithm*}

%% file: lazyRenormalization.tex
\section{The Star Network Model}
\label{sec:lazy-renorm}

In some applications, the assumption that sensors can broadcast messages to all sensors may be unrealistic.  Furthermore, in some applications sensors may not be able to compute the marginal benefits $\obj_{t}(S\cup\{s\})-\obj_{t}(S)$ (since this calculation may be computationally complex).  In this section, we analyze \lazydog, a variant of our \dog algorithm, which replace the above assumptions by the assumption that there is a dedicated base station\footnote{Though the existence of such a base station means the protocol is not
completely distributed, it is realistic in sensor network applications
where the sensor data needs to be accumulated somewhere for analysis.}
 available which  computes utilities and which can send non-broadcast messages to individual sensors.

We make the following assumptions:
\begin{enumerate}\denselist
\item Every sensor stores its probability mass $p_\sensor$ with
it, and can only send messages to and receive messages from the base station.  
\item The base station is able, after receiving messages from a set $S$ of sensors, to compute the utility $\obj_{t}(S)$ and send this utility back to the active sensors.  
\end{enumerate}

These conditions arise, for example, when cell phones in
participatory sensor networks can contact the base station, but due to
privacy constraints cannot directly call other phones.  We do not
assume that the base station has access to all weights of the sensors
-- we will only require the base station to have $\cO(\K + \log n)$
memory.  In the fully distributed algorithm \dog that relies on
broadcasts, it is easy for the sensors to maintain their normalizers
$Z_{\sensor,i}$, since they receive information about rewards from all
selected sensors. The key challenge when removing the broadcast
assumption is to maintain the normalizers in an appropriate manner.

\vspace{-1mm}
\subsection{Lazy renormalization \& Distributed EXP3} \label{ssec:lazy}
\EXPthree (and all MAB with no-regret guarantees against arbitrary
reward functions) must maintain a distribution over actions, and
update this distribution in response to feedback about the
environment.
In \EXPthree, each sensor $\sensor$ requires only $w_{\sensor}(t)$ and
a normalizer $Z(t) :=\sum_{\sensor'}w_{\sensor'}(t)$ to compute
$p_{\sensor}(t)$\footnote{We let $x(t)$ denote the value of variable
$x$ at the start of round $t$, to ease analysis.  We do not actually
need to store the historical values of the variables over multiple time steps.}.  The
former changes only when $\sensor$ is selected. In the broadcast model
the latter can simply be broadcast at the end of each round. In the
star network model (or, more generally in multi-hop models), standard
flooding echo aggregation techniques could be used to compute and
distribute the new normalizer, though with high communication cost. We
show that a lazy renormalization scheme can significantly reduce the
amount of communication needed by a distributed bandit algorithm
\emph{without altering its regret bounds whatsoever}. Thus our lazy
scheme is complementary to standard aggregation techniques.
\ignore{ In the
broadcast model the latter can simply be broadcast at the end of each
round, however things are not so easy in the star network model.
Nevertheless, in this section we show that for \EXPthree (and many
algorithms like it), we can perform the necessary renormalization in a
lazy manner that keeps the amount of communication quite low.
}

Our lazy renormalization scheme for \EXPthree works as follows.  Each
sensor $\sensor$ maintains its weight $w_{\sensor}(t)$ and an estimate
$Z_{\sensor}(t)$ for $Z(t) := \sum_{\sensor'}w_{\sensor'}(t)$,
Initially, $w_{\sensor}(0) = 1$ and $Z_{\sensor}(0) = \numsensors$ for
all $\sensor$.  The central server stores $Z(t)$.
Let 
\[
\rho(x, y) := (1-\gamma)\frac{x}{y} + \frac{\gamma}{\numsensors}\mbox{.}
\]
Each sensor then proceeds to activate as in the sampling procedure 
of~\secref{sec:sampling} as if its probability mass in round $t$
were $q_{\sensor} = \rho(w_{\sensor}(t), Z_{\sensor}(t))$ instead of its true
value of $\rho(w_{\sensor}(t), Z(t))$. A single sensor is selected by
the server with respect to the true value $Z(t)$, resulting in a
selection from the desired distribution. Moreover, $\sensor$'s estimate 
$Z_{\sensor}(t)$ is only updated on rounds when it communicates with
the server under these circumstances.  This allows the estimated
probabilities of all of the sensors to sum to more than one, but has
the benefit of significantly reducing the communication cost in the
star network model under certain assumptions.  
We call the result \emph{Distributed EXP3}, give its pseudocode 
for round $t$ in~\figref{fig:lazy-renorm-code}.

Since the sensors underestimate their normalizers, they may activate more
frequently than in the broadcast model.   Fortunately, the amount of
``overactivation'' remains bounded.
\ifthenelse{\boolean{istechrpt}}{
We prove Theorem~\ref{thm:dexp3-performance} and
Corollary~\ref{cor:dexp3-performance} in Appendix~\ref{sec:lazy-appendix}. 
}{ %
We prove Theorem~\ref{thm:dexp3-performance} and
Corollary~\ref{cor:dexp3-performance} in the full version 
of this paper~\cite{golovin09distributed}.
} 
\begin{theorem}\label{thm:dexp3-performance}
The number of sensor activations in any round of the Distributed \EXPthree algorithm 
is at most \mbox{$\alpha + (e-1)$} in
expectation and $\cO(\alpha + \log n)$ with high probability, and
the number of messages is at most twice the number of activations.
\end{theorem}

Unfortunately, there is still an $e^{-\alpha}$ probability of nothing
being selected.  To address this, we can set $\alpha = c\ln n$ for
some $c \ge 1$, and if 
nothing is selected, transmit a message to each of
the $n$ sensors to rerun the protocol.  

\begin{cor}\label{cor:dexp3-performance}
There is a distributed implementation of \EXPthree that always selects a
sensor in each round, has the same regret bounds as standard
\EXPthree, ensures that  
the number of sensor activations in any round is 
at most $\ln n + \cO(1)$ in expectation or $\cO(\log n)$ with high
probability, and in which the number of messages is at most twice the number of
activations.
\end{cor}

\begin {algorithm}[tb]
\Titleofalgo { Distributed \EXPthree (executing on round $t$)}
\KwIn{Parameters $\alpha, \eta, \gamma \in \reals_{> 0}$, sensor set $\sensors$.}
Let $\rho(x, y) := (1-\gamma)\frac{x}{y} + \frac{\gamma}{|\sensors|}$.\\
\emph{Sensors: %
} \\
\For {\KwEach \emph{sensor $\sensor$ in parallel}}{
  Sample $r_{\sensor}$ uniformly at random from $[0,1]$.\\
  \If {$(r_{\sensor} \ge 1 - \alpha \cdot \rho(w_{\sensor}(t), Z_{\sensor}(t))$}{
    Send $\tuple{r_{\sensor}, w_{\sensor}(t)}$ to the server.\\
    Receive message $\tuple{Z, w}$ from server.\\
    $Z_{\sensor}(t+1) \gets Z$; $w_{\sensor}(t+1) \gets w$. \\
  }\lElse{
     $Z_{\sensor}(t+1) \gets Z_{\sensor}(t)$; $w_{\sensor}(t+1) \gets w_{\sensor}(t)$.\\
  }
}
\emph{Server:}\\
 Receive messages from a set $S$ of sensors. \\
 \lIf{$S = \emptyset$}{Select nothing and wait for next round.}
 \lElse{
     \For {\KwEach \emph{sensor} $\sensor \in S$}{
       $Y_{\sensor} \gets \min \set{x : \prob{X \le x} \ge r_{\sensor}}$, 
       where $X \sim \Poisson{\alpha \cdot \rho(w_{\sensor}(t),Z(t))}$.\\
       Select $\sensor$ with probability
$Y_{\sensor}/\sum_{\sensor' \in S} Y_{\sensor'}$.\\
       Observe the payoff $\payoff$ for the selected
       sensor $\sensor^*$;  
       $w_{\sensor^*}(t+1) \gets w_{\sensor^*}(t) \cdot \exp
       \set{\eta \payoff /  \rho(w_{\sensor^*}(t),Z(t))}$;
       $Z(t+1) \gets Z(t) + w_{\sensor^*}(t+1) -  w_{\sensor^*}(t)$\;
       \lFor{ \KwEach $\sensor \in S \setminus \sensor^*$}{
          $w_{\sensor}(t+1) \gets w_{\sensor}(t)$\;
       }
       \lFor{ \KwEach $\sensor \!\in\! S$}{
         Send $\tuple{Z(t\!+\!1), w_{\sensor}(t\!+\!1)}$ to $\sensor$.\\
    }
  }
}
\caption{Distributed \EXPthree: the \LimitProtocolShortns($\alpha$) with lazy renormalization,
  applied to \EXPthree}
\label{fig:lazy-renorm-code}
\end {algorithm}

\vspace{-2mm}
\subsection{LazyDOG} \label{sec:lazy-dog}
Once we have the distributed \EXPthree variant described above, we can
use it for the bandit subroutines in the \OGunit algorithm
(\cf~\secref{sec:centralized-multi}).  We call the result the \lazydog
algorithm, due to its use of lazy renormalization. The lazy distributed
\EXPthree still samples sensors from the same distribution as the
regular distributed \EXPthree, so \lazydog has precisely the same
performance guarantees with respect to $\sum_{t} \obj_t(S_t)$ as
\dog. It works in the star network communication model, and requires few
messages or sensor activations.  Corollary~\ref{cor:dexp3-performance}
immediately implies the following result.

\begin{cor}\label{cor:lazydog-performance}
The number of sensors that activate each round in 
\lazydog is at most $\K\ln n + \cO(\K)$ in expectation and 
$\cO(\K\log n)$ with high probability, 
the number of messages is at most twice the number of activations, and
the $\paren{1 - 1/e}$-regret of \lazydog is the same as \dog.
\end{cor}

\ifthenelse{\boolean{istechrpt}}{
If we are not concerned about the exact number of sensors selected in
each round, but only want to ensure roughly $\K$ sensors are picked in
expectation, then we can reduce the number of sensor activations and
messages to $\cO(\K)$, by running \lazydog with $\K' := \ceil{\K/(1 -
  e^{-\alpha})}$ stages for some constant $\alpha$, and allowing each stage to run the
\LimitProtocol with lazy renormalization \emph{without} rerunning it if nothing is selected.
This is of course optimal up to constants, as we must send at least
one message per selected sensor.

\begin{theorem} \label{thm:better-lazydog}
  The variant of \lazydog that runs the \LimitProtocol$(\alpha)$ with lazy
  renormalization for $\K' := \ceil{\K/(1 - e^{-\alpha})}$ stages, but
  does not rerun it if nothing is selected in a given stage, has
  the following guarantees: ($1$) the number of sensors that activate each round in 
\lazydog is at most 
$\K' (\alpha + e -1)$ in expectation and 
$\cO(\alpha\K\log n)$ with high probability, ($2$) the number  of messages
is at most twice the number of activations, ($3$) the expected number
of sensors selected in each round is at most $\K'$
and $(4)$ its $\paren{1 - 1/e}$-regret is at most $\K' / \K$ times that of \dog.
\end{theorem}
\noindent
We defer the proof to Appendix~\ref{sec:faulty}.

} {        
}

%% file: measurementDependentSampling.tex
\section{Observation-Dependent Sampling} \label{sec:measure-dependent-sampling}

Theorem~\ref{thm:dog-performance} states that
\algoname is guaranteed to do nearly as well as the offline greedy
algorithm run on an instance with objective function $\obj_{\Sigma} :=
\sum_{t} \obj_t$.  Thus the reward of \algoname is asymptotically near-optimal on average.
In many applications, however, we would like to perform well on rounds
with ``atypical'' objective functions.  For example, in an outbreak detection application as we discuss in \secref{sec:experiments}, we would like to get very good data on rounds with
significant events, even if the nearest sensors typically
report ``boring'' readings that contribute very little to the
objective function.   For now, suppose that we are only running a
single MAB instance to select a single sensor in each round.
If we have access to 
a black-box for evaluating $\obj_t$ on round $t$, then 
we can perform well on atypical rounds at the cost of some additional
communication by having each sensor $\sensor$ take a local reading of its
environment and estimate its payoff $\estimate{\payoff}=\obj_{t}(\{v\})$ if
selected.  This value, which serves as a measure of how interesting
its input is, can then be used to decide whether to boost $\sensor$'s probability for
reporting its sensor reading to the server.
In the simplest case, we can imagine that each $\sensor$ has a threshold 
$\thresh_{\sensor}$ such that $\sensor$ activates with probability $1$
if $\estimate{\payoff} \ge \thresh_{\sensor}$, and with its normal
probability otherwise.  
In the case where we select $\K > 1$ sensors
in each round, each sensor can have a threshold for each of the $\K$
stages, where in each stage it computes $\estimate{\payoff}=\obj_{t}(S \cup \{v\}) -
\obj_{t}(S)$ where $S$ is the set of currently selected sensors.
Since the activation probability only goes up, 
we can retain the performance guarantees of \dog if we are
careful to adjust the feedback properly.

Ideally, we wish that the sensors learn what their thresholds
$\thresh_{\sensor}$ should be.  
\ignore{
Consider that in many physical sensing
applications the marginal value derived from a sensor's
reading depends strongly on what other nearby sensors observe and
whether they report their readings to the central server.
If a set of sensors $S$ typically have redundant readings, we would
like a few sensors in $S$ to have low thresholds, and most other
sensors in $S$ to have high thresholds, so that in expectation few
sensors in $S$ report their (redundant) readings in any given round.
In contrast, setting a fixed threshold $\thresh$ for all sensors in
$S$ means that whenever one sensor in $S$ detects a sufficiently
interesting event, all sensors in $S$ report to the server, thus leading to a large number of unnecessary messages.
} %
We treat the selection of $\thresh_{\sensor}$ in each round as an
online decision problem that each $\sensor$ must play.
We construct a particular game that the sensors play, where the
strategies are the thresholds (suitably discretized), there is an  
\emph{activation cost} $\actcost_{\sensor}$ that $\sensor$ pays if 
$\estimate{\payoff}_{\sensor} \ge \thresh_{\sensor}$, and the 
payoffs are defined as follows:
Let $\payoff_{\sensor} = \obj_{t}(S\cup \{v\})-\obj_{t}(S)$ be the
marginal benefit of selecting $\sensor$ given that sensor set $S$ has
already been selected.
Let $A$ be the set of sensors that activate in the current iteration of
the game, and let $\paymaxact{\sensor} := \max \paren{\payoff_{\sensor'}
    : \sensor' \in A \setminus \set{\sensor}}$.
The particular reward function $\gpay_{\sensor}$ we 
choose for each sensor $\sensor$  for each iteration of the game is 
\[
\gpay_{\sensor}(\thresh) = \left\{
\begin{array}{ll}
\actcost_{\sensor} - \max\paren{\payoff_{\sensor} - \paymaxact{\sensor}, 0} & \text{if
}\estimate{\payoff} < \thresh\\
\max\paren{\payoff_{\sensor} - \paymaxact{\sensor}, 0} - \actcost_{\sensor} & \text{if
}\estimate{\payoff} \ge \thresh
\end{array}
\right.
\]
\ignore{
\[
\gpay_{\sensor}(\thresh) = \left\{
\begin{array}{ll}
\actcost_{\sensor} - \paren{\payoff_{\sensor} - \paymaxact{\sensor}} & \text{if
}\estimate{\payoff} < \thresh\\
\paren{\payoff_{\sensor} - \paymaxact{\sensor}} - \actcost_{\sensor} & \text{if
}\estimate{\payoff} \ge \thresh
\end{array}
\right.
\]
}
based on empirical performance. Thus, if a sensor activates ($\estimate{\payoff} \ge \thresh$), its payoff is the improvement over the best payoff $\payoff_{\sensor'}$ among all sensors $\sensor'\in A$ minus its activation cost. In case multiple sensors activate, the highest reward is retained.

In the broadcast model where each sensor can compute its marginal
benefit, we can use any standard no-regret algorithm for combining
expert advice, such as \emph{Randomized} \emph{Weighted}
\emph{Majority} ($\WMR$)~\cite{littlestone94}, to play this game and
obtain no regret guarantees\footnote{We leave it as an open problem to
determine if the outcome is close to optimal when all sensors play low
regret strategies (i.e., is the \emph{price of total
anarchy}~\cite{blum08} small in any variant of this game with a
reasonable way of splitting the value from the information?)}  for
selecting $\thresh_{\sensor}$.  In our context a sensor using $\WMR$
simply maintains weights $w(\thresh_i) = \exp \paren{\eta \cdot
\gpay_{\text{total}}(\thresh_i)}$ for each possible threshold
$\thresh_i$, where $\eta > 0$ is a learning parameter, and
$\gpay_{\text{total}}(\thresh_i)$ is the total cumulative reward for
playing $\thresh_i$ in every round so far.  On each step each
threshold is picked with probability proportional to its weight.  In
the more restricted star network model, we can use a modification of
\WMR that feeds back unbiased estimates for $\gpay_{t}(\thresh_i)$,
the payoff to the sensor for using a threshold of $\thresh_i$ in round
$t$, and thus obtains reasonably good estimates of
$\gpay_{\text{total}}(\thresh_i)$ after many rounds.  We give
pseudocode in~\figref{fig:modifiedWMR}.  In it, we assume that an
activated sensor can compute the reward of playing any threshold.

\begin{algorithm} \label{fig:modifiedWMR}
\Titleofalgo{ Modified \WMR (star network setting)}
\KwIn{parameter $\eta > 0$, threshold set $\set{\thresh_i : i \in [m]}$}
Initialize $w(\thresh_i) \gets 1$ for all $i \in [m]$.\\
\For{ \KwEach\emph{ round} $t = 1, 2, \ldots$}{
  Select $\thresh_i$ with probability $w(\thresh_i) / \sum_{j=1}^m
  w(\thresh_j)$.\\
  \If{\emph{ sensor activates}}{
     Let $\gpay(\thresh_i)$ be the reward for playing $\thresh_i$ in this round of
     the game.  Let $q(\thresh_i)$ be the total probability of activation
     conditioned on $\thresh_i$ being selected (including the
     activation probability that does not depend on local
     observations.)\\
     \For{ \KwEach\emph{ threshold $\thresh_i$}}{
       $w(\thresh_i) \gets w(\thresh_i)\exp\paren{\eta \gpay(\thresh_i)/q(\thresh_i)}$.\\
     }
  }
}
\caption{Selecting activation thresholds for a sensor}
\end{algorithm}

We incorporate these ideas into the \algoname algorithm, to obtain
what we call the \emph{Observation-Dependent}
\emph{Distributed Online} \emph{Greedy} algorithm (\oddog).
In the extreme case that $\actcost_{\sensor} = 0$ for all
$\sensor$ the sensors will soon set their thresholds so low that each sensor activates in
each round.  In this case \oddog will exactly simulate the offline
greedy algorithm run on each round.
In other words, if we let $\gout(f)$ be the result of running the  offline greedy
algorithm on the problem 
\[\argmax\set{f(S) : S \subset \sensors,\ |S| \le \K}\] 
then \oddog will obtain a value of $\sum_t
\obj_t(\gout(\obj_t))$;  in contrast, \dog gets roughly 
$\sum_t \obj_t(\gout(\sum_t \obj_t))$, which may be significantly smaller.
Note that Feige's result~\cite{feige97} implies that the former value is the
best we can hope for from efficient algorithms (assuming $\P \neq
\NP$).
Of course, querying each sensor in each round 
is impractical when querying sensors is
expensive.  In the other extreme case where $\actcost_{\sensor} = \infty$ for all
$\sensor$, \oddog will simulate \dog after a brief learning phase.
In general, by adjusting the activation costs $\actcost_{\sensor}$
we can smoothly trade off the cost of sensor communication with the value of the resulting data.

\ignore{
That is, it obtains a total value of almost
$\obj_{\Sigma}(\gout(\obj_{\Sigma}))$, where 
$\gout(f)$ is the result of running the simple offline greedy
algorithm on the problem 
\[\argmax_{ S \subset \sensors,\ |S| \le \K}\set{f(S)}\mbox{.}\] 
Ideally, we might hope to do better by using more communication and
exploiting local sensing in the algorithm.

Feige's result~\cite{feige97} implies that the best we can hope for from
efficient algorithms (assuming $\P \neq \NP$) 
is to do as well as the offline greedy algorithm
run on each round, i.e., $\sum_t \obj_t(\gout(\obj_t))$.
With unlimited communication and an value oracle for each $\obj_t$, we
can indeed obtain this much value simply by querying every sensor and then running the
offline greedy algorithm in each round.  
}
\ignore{
We suppose that on each round $t$ every sensor $\sensor$ takes a reading of its
environment and uses it to compute, 
for each MAB instance $I$ it participates in, an estimate of its
payoff $\estimate{\payoff} = \estimate{\payoff}(\sensor,t,I)$ if it is selected.  
Then when running the \LimitProtocolShort with lazy renormalization for
instance $I$, $\sensor$ boosts its activation probability by a
multiplicative factor of $\boost_{\sensor, I}(\estimate{\payoff})$, 
where $\boost_{\sensor, I}:[0,1] \to [1, \infty)$ is a function
that $\sensor$ learns to optimize its contribution to the overall solution.
}

%% file: experiments.tex
\section{Experiments}\label{sec:experiments}
In this section, we evaluate our \dog algorithm on several real-world sensing problems.

\subsection{Data sets}
\vspace{-0.5em}

\paragraph{Temperature data}  
In our first data set, we analyze temperature measurements from the
network of 46 sensors deployed at Intel Research Berkeley. Our
training data consisted of samples collected at 30 second intervals on 3
consecutive days (starting Feb. 28th 2004), the testing data consisted
of the corresponding samples on the two following days. The objective
functions used for this application are based on the expected
reduction in mean squared prediction error $\fEMSE$, as introduced in
\secref{sec:problem}.

\paragraph{Precipitation data}
Our second data set consists of precipitation data collected during
the years 1949 - 1994 in the states of Washington and Oregon
\citep{raindata}. Overall 167 regions of equal area, approximately 50
km apart, reported the daily precipitation. To ensure the data could
be reasonably modeled using a Gaussian process we applied preprocessing as
described in \cite{krause07jmlr}. As objective functions we again use
the expected reduction in mean squared prediction error $\fEMSE$.

\paragraph{Water network monitoring}
Our third data set is based on the application of monitoring for
outbreak detection.  Consider a city water distribution network for
delivering drinking water to households. Accidental or malicious
intrusions can cause contaminants to spread over the network, and we
want to install sensors to detect these contaminations as quickly as
possible. In August 2006, the Battle of Water Sensor Networks (BWSN)
\cite{ostfeld08bwsn} was organized as an international challenge to
find the best sensor placements for a real metropolitan water
distribution network, consisting of 12,527 nodes. In this challenge, a
set of intrusion scenarios is specified, and for each scenario a
realistic simulator provided by the EPA is used to simulate the spread
of the contaminant for a 48 hour period. An intrusion is considered
detected when one selected node shows positive contaminant
concentration. The goal of BWSN was to minimize impact measures, such
as the expected population affected, which is calculated using a
realistic disease model. For a security-critical sensing task such as
protecting drinking water from contamination, it is important to
develop sensor selection schemes that maximize detection performance
even in adversarial environments (i.e., where an adversary picks the
contamination strategy knowing our network deployment and selection
algorithm). The algorithms developed in this paper apply to such
adversarial settings. We reproduce the experimental setup detailed in
\cite{krause08efficient}. For each contamination event $i$, we define
a separate submodular objective function $f_{i}(S)$ that measures the
expected population protected when detecting the contamination from
sensors $S$.  In \cite{krause08efficient}, Krause et al.~showed that
the functions $\obj_{i}(A)$ are monotone submodular functions.

\subsection{Convergence experiments}
In our first set of experiments, we analyzed the convergence of our
\dog algorithm.  For both the temperature [T] and precipitation [R]
data sets, we first run the offline greedy algorithm using the
$\fEMSE$ objective function to pick $k=5$ sensors.  We compare its
performance to the \dog algorithm, where we feed back the same
objective function at every round.  We use an exploration probability
$\gamma=0.01$ and a learning rate inversely proportional to the
maximum achievable reward $\fEMSE(\sensors)$.  \figref{fig:tempconv}
presents the results for the temperature data set.  Note that even
after only a small number of rounds ($\approx 100$), the algorithm
obtains 95\% of the performance of the offline algorithm.  After about
13,000 iterations, the algorithm obtains 99\% of the offline
performance, which is the best that can be expected with a $.01$
exploration probability. \figref{fig:rainconv} show the same
experiment on the precipitation data set.  In this more complex
problem, after 100 iterations, 76\% of the offline performance is
obtained, which increases to 87\% after 500,000 iterations.

\ignore{
  At each round, a different objective
function $f_{t}$ is selected, one for each contamination event.  We 
ran the offline greedy algorithm on the objective
$f_{\Sigma}(S)=\sum_{t}f_{t}(S)$, and then compared its result with
our \dog algorithm.  \figref{fig:waterconv} presents the results of
this experiment.  Note that in contrast to the results in
\figref{fig:tempconv} and \figref{fig:rainconv}, where we always feed
back the same objective function, in the varying objective function
case, there is more variance in the performance. However, 
despite these varying objectives, the average performance still
converges to the value of the offline solution that optimizes the
selection given knowledge of all objective functions.
}
\newcommand{\figwidth}{0.48\textwidth}
\newcommand{\figheightA}{0.22\textwidth}
\newcommand{\figheightB}{0.24\textwidth}

\subsection{Observation dependent activation}

We also experimentally evaluate our \oddog algorithm with observation
specific sensor activations.  We choose different values for the
activation cost $c_{v}$, which we vary as multiples of the total
achievable reward.  The activation cost $c_{v}$ lets us smoothly trade
off the average number of sensors activating each round and the
average obtained reward. The resulting activation strategies are used
to select a subset of size $k=10$ from a collection of 12,527 sensors.
\figref{fig:waterfixed} presents rates of convergence using the \oddog
algorithm under a fixed objective function which considers all
contamination events. In \figref{fig:watervarying}, convergence rates
are presented under a varying objective function, which selects a
different contamination event on each round. For low activation costs,
the performance quickly converges to or exceeds the performance of the
offline solution. Even under the lowest activation costs in our
experiments, the average number of extra activations per stage in the
\oddog algorithm is at most 5. These results indicate that observation
specific activation can lead to drastically improved performance at
small additional activation cost.

\begin{figure*}%
\centering 
\subfigure[\emph{[T] Convergence}]{
\includegraphics[width=\figwidth,height=\figheightB]{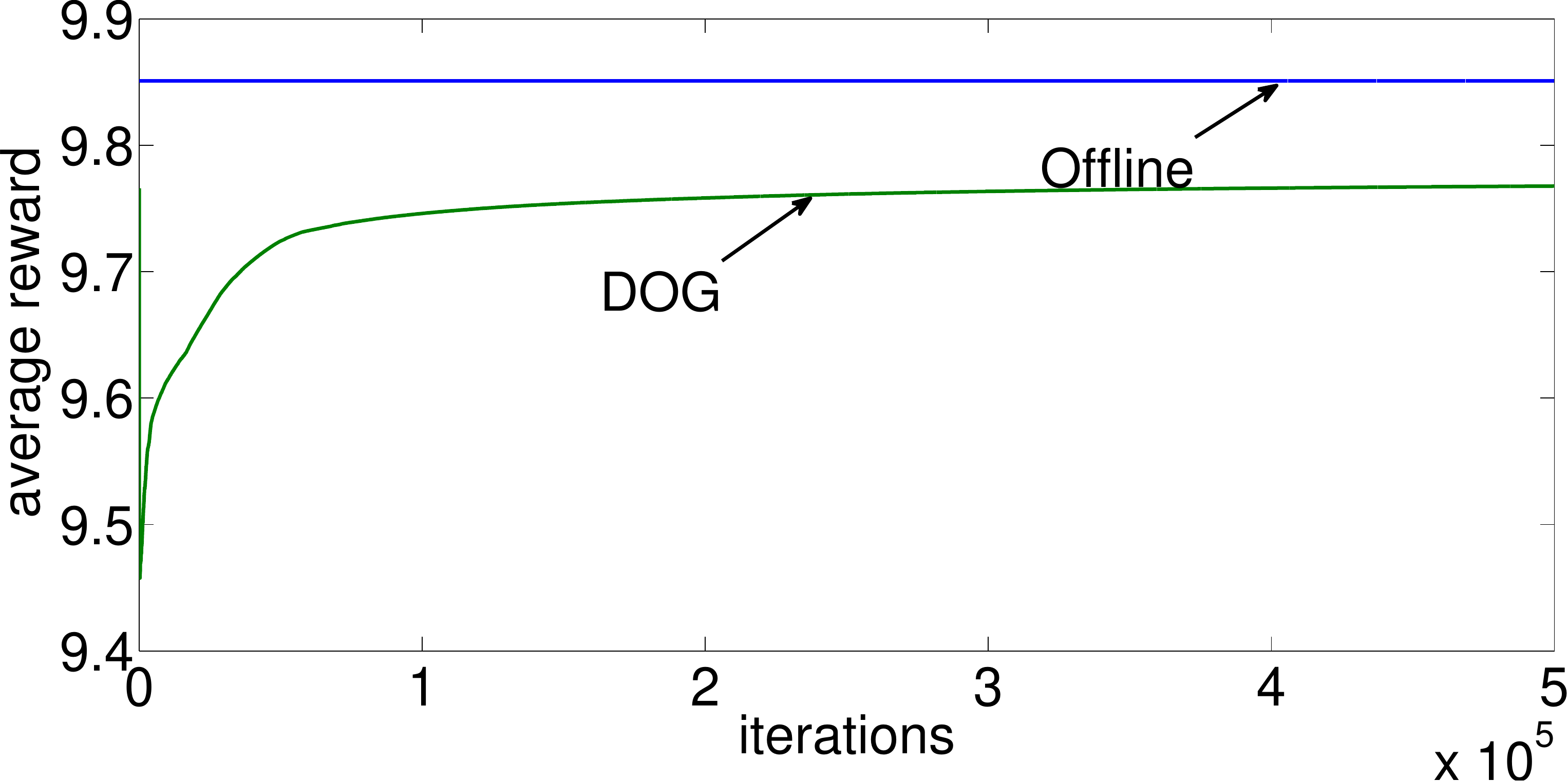}
 \label{fig:tempconv}
 }
\subfigure[\emph{[R] Convergence}]{
\includegraphics[width=\figwidth,height=\figheightB]{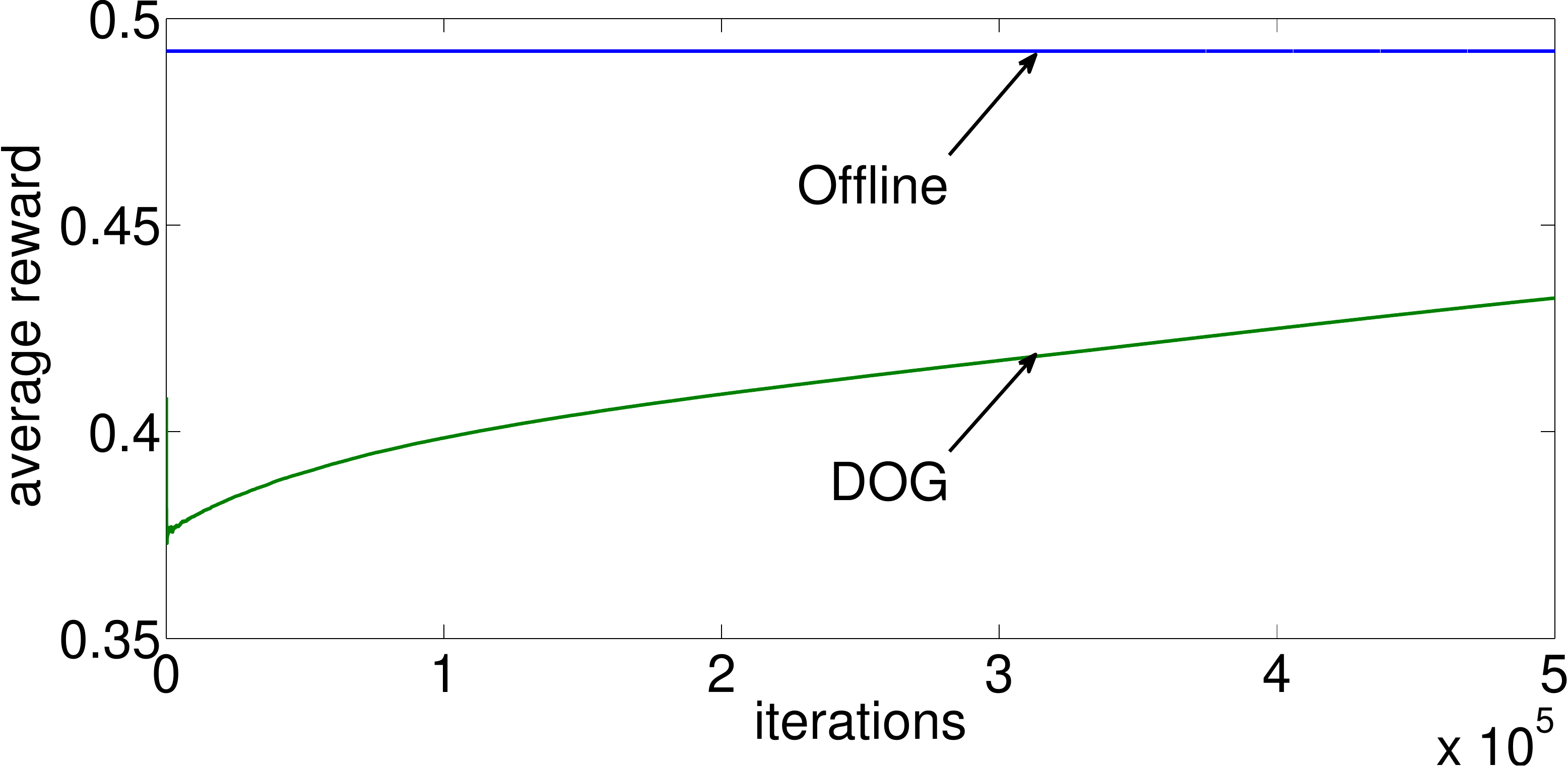}
 \label{fig:rainconv}
 }
 \subfigure[\emph{[W] Constant Objective}]{
\includegraphics[width=\figwidth,height=\figheightB]{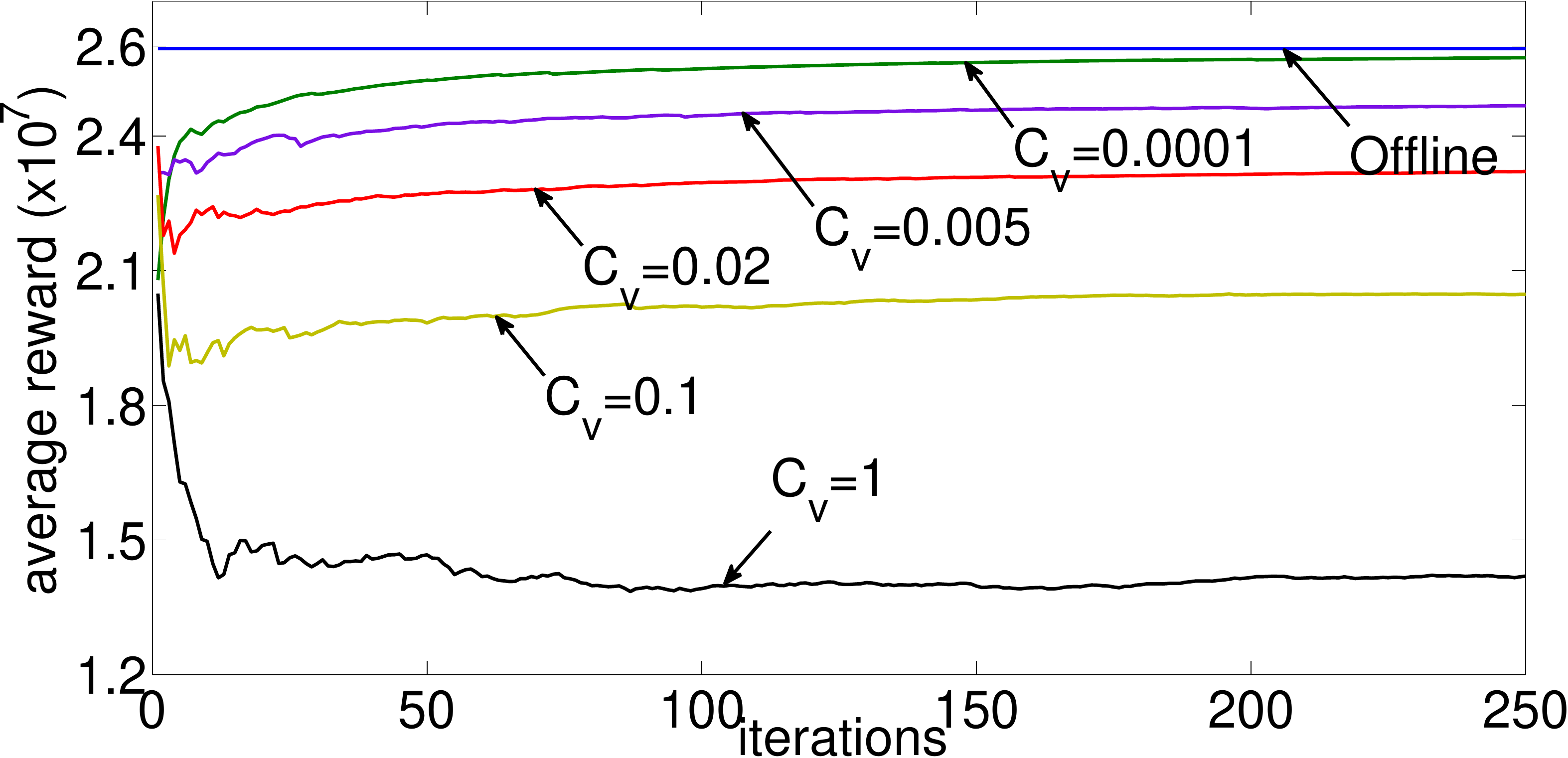}
 \label{fig:waterfixed}
 }
\subfigure[\emph{[W] Varying Objective}]{
\includegraphics[width=\figwidth,height=\figheightB]{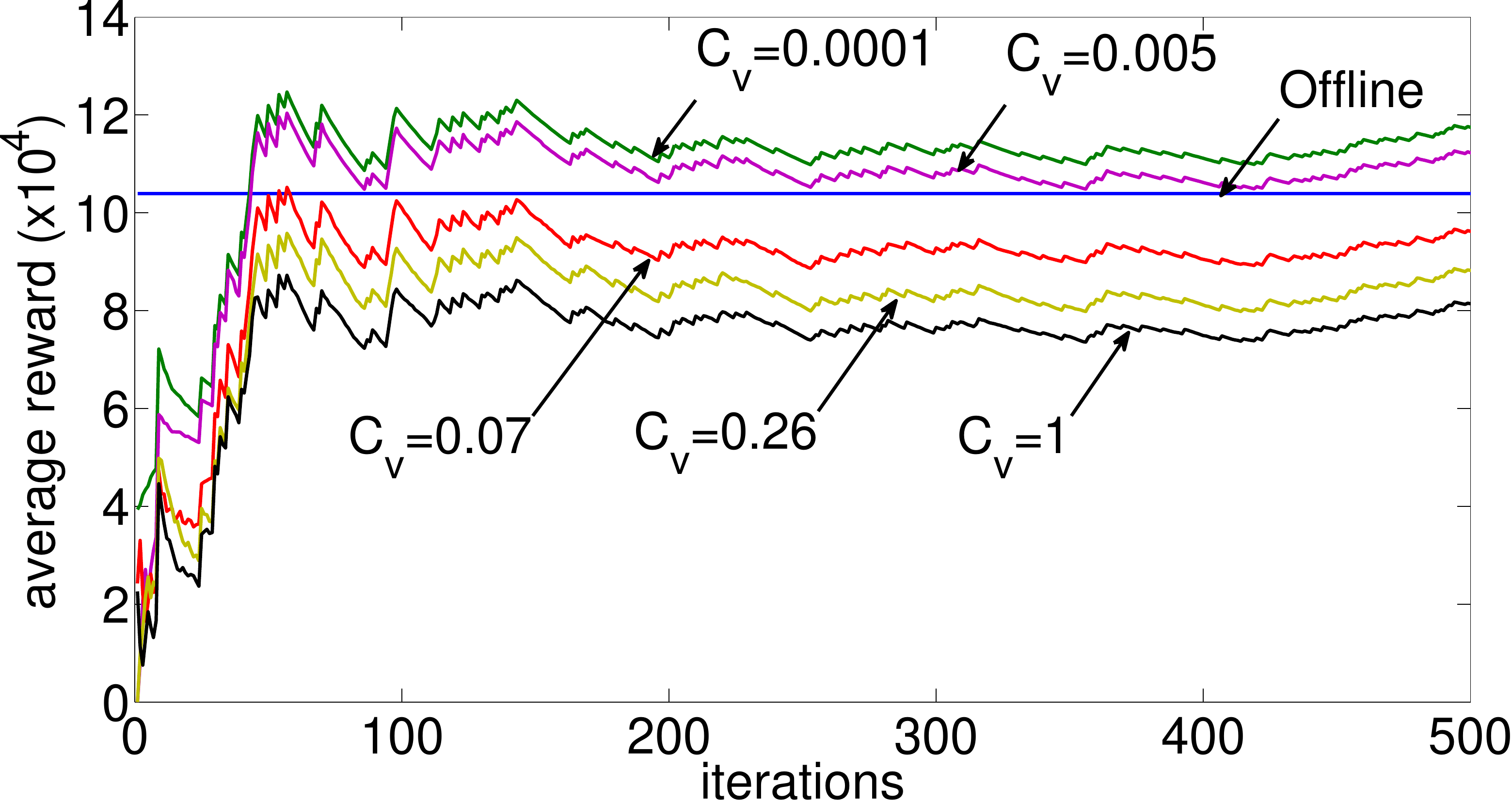}
 \label{fig:watervarying}
 }
\caption{\footnotesize Experimental results on [T]
 Temperature data, [R] precipitation data and [W] water distribution
 network data.
 \label{fig:convergence}
 }%
\end{figure*}

%% file: related.tex
\section{Related Work}\label{sec:related}\vspace{-3mm}
\paragraph{Sensor Selection}
The problem of deciding when to selectively turn on sensors in sensor
networks in order to conserve power was first discussed by
\citet{slijepcevic01power} and \citet{zhao02IDSQ}. Many approaches for
optimizing sensor placements and selection assume that sensors have a
fixed region
\citep{Hochbaum+Maas:1985,Gonzalez:2001,kumar06disk}. These regions
are usually convex or even circular. Further, it is assumed that
everything within a region can be perfectly observed, and everything
outside cannot be measured by the sensors. For complex applications
such as environmental monitoring, these assumptions are unrealistic,
and the direct optimization of prediction accuracy is desired. The
problem of selecting observations for monitoring spatial phenomena has
been investigated extensively in geostatistics~\citet{Cressie:1991},
and more generally (Bayesian) experimental design
\cite{Chaloner1995}. Several approaches have been proposed to
activate sensors in order to minimize uncertainty \cite{zhao02IDSQ} or
prediction error \cite{vldb04}. However, these approaches do not have
performance guarantees.  Submodularity has been used to analyze
algorithms for placing \citep{krause07jmlr} or selecting
\cite{williams07performance} a fixed set of sensors. These approaches
however assume that the model is known in advance.

\paragraph{Submodular optimization}
The problem of centralized maximization of a submodular function has
been studied by \citet{nemhauser78}, who proved that the greedy
algorithm gives a factor $(1-1/e)$ approximation. Several
algorithms have since been developed for maximizing submodular functions
under more complex constraints (see \cite{vondrak07} for
an overview).  Streeter and Golovin developed an algorithm for online
optimization of submodular functions, which we build on in this paper
\cite{streeter08}.

%% file: conclusions.tex
\section{Conclusions}

In this paper, we considered the problem of repeatedly selecting
subsets $S_{t}$ from a large set of deployed sensors, in order to
maximize a sequence of submodular utility functions
$\obj_{1},\dots,\obj_{T}$. We developed an efficient Distributed
Online Greedy algorithm \dog, and proved it suffers no
$(1-1/e)$-regret, essentially the best possible performance obtainable
unless P = NP.  Our algorithm is fully distributed, requiring only a
small number of messages to be exchanged at each round with high
probability.  We analyze our algorithm both in the broadcast model,
and in the star network model, where a separate base station is
responsible for computing utilities of selected sets of sensors.  Our
\lazydog algorithm for the latter model uses lazy renormalization in
order to reduce the number of messages required from $\Theta(n)$ to
$\cO(\K\log n)$, and the server memory required from $\Theta(n)$ to
$\cO(\K + \log n)$, where $\K$ is the desired number of sensors to be
selected.  In addition, we developed \oddog, an extension of \dog that
allows observation-dependent sensor selection.  We empirically
demonstrate the effectiveness of our algorithms on three real-world
sensing tasks, demonstrating how our \dog algorithm's performance
converges towards the performance of a clairvoyant offline greedy
algorithm.  In addition, our results with the \oddog algorithm
indicate that a small number of extra sensor activations can lead to
drastically improved convergence.  We believe that our results provide
an interesting step towards a principled study of distributed active
learning and information gathering.

\paragraph{Acknowledgments}  The authors wish to thank Phillip Gibbons
and the anonymous referees for their valuable help and
suggestions. This research was partially supported by ONR grant
N00014-09-1-1044, NSF grant CNS-0932392, a gift from Microsoft
Corporation and the Caltech Center for the Mathematics of Information.

%% file: appendices.tex
\ifthenelse{\boolean{istechrpt}}{
\appendix

\section{Results in the Broadcast Model} \label{sec:activation-appendix}

\begin{proof}[of Theorem~\ref{thm:distributed-single-sensor-selection}]
To prove the regret bounds, note that in every round the distribution over sensor
selections in the variant of \EXPthree we describe 
(that uses the distributed multinomial sampling scheme and repeatedly reruns the protocol in
order to always select some sensor in each round) is precisely the
same as the original \EXPthree.  Thus the regret bounds for
\EXPthree~\cite{auer02} carry over unchanged.
We next bound the number of broadcasts.  Fix a round, and let $S$ set
of sensors that activate in that round.  The total number of broadcasts is
then $|S|+2$; using their calibrated clocks, each sensor (re)samples  
$X_{\sensor} \sim \Poisson{\alpha p_{\sensor} }$ and activates if
$X_{\sensor} \ge 1$.  If no sensors activate before a specified
timeout period, the default behavior is to rerun the sampling step.
Eventually $|S| \ge 1$ sensors activate in the same period.  A
distinguished sensor in $S$ then determines the selected sensor $\sensor$,
broadcasts $\id(\sensor)$, and $\sensor$ broadcasts its observed reward.
We prove $\E{|S|} \le \alpha / (1 - e^{-\alpha})$ in
Proposition~\ref{prop:broadcast-activations}.  When $\alpha = 1$, this
gives us the claimed bound on the number of broadcasts.
\end{proof}

\begin{prop} \label{prop:broadcast-activations}
Rerunning the \LimitProtocol until an
element is selected results in at most  $\alpha / (1 - e^{-\alpha})$
elements being activated in expectation.  Moreover, this value is tight.
\end{prop}

\begin{proof}
Let $X_{\sensor} \sim \Bernoulli{\alpha \cdot p_{\sensor}}$ be the indicator random
variable for the activation of $\sensor$, and let $X := \sum_{\sensor}
X_{\sensor}$.  The expected number of sensor activations is then
\[
\E{X \ |\ X \ge 1} = \E{X}/\prob{X \ge 1} \mbox{.}
\]
In the limit as $\max_{\sensor} p_{\sensor}$ tends to zero, $X$
converges to a Poisson random variable with mean $\alpha$.
In this case, $\frac{\E{X}}{\prob{X \ge 1}} = \alpha / (1 - e^{-\alpha})$
To see that this is an upper bound, consider an arbitrary distribution
$p$ on the sensors, and fix some $\sensor$ with $x := p_{\sensor} > 0$.
We claim that replacing $\sensor$ with two sensors $\sensor_1$ and
$\sensor_2$ with positive probability mass $x_1$ and
$x_2$ with $x = x_1+x_2$ can only serve to increase the expected number of
sensor activations, because $\E{X}$ is unchanged, and $\prob{X \ge 1}$
decreases.  The latter is true essentially because 
$\prob{\exists i \in \set{1,2}: \sensor_i \text{ activates}} = 1 -
(1-x_1)(1-x_2) = x - x_1x_2 < x$.  To complete the proof, notice that 
repeating this process with $\sensor = \argmax (p_{\sensor})$ and $x_i = x/2$ ensures $X$
converges to a Poisson variable with mean $\alpha$, while only
increasing $\E{X \ |\ X \ge 1}$.
\ignore{
easily verified by noting that conditioning
on the event $\event$ that no sensors in $\sensors \setminus
\set{\sensor, \sensor_1, \sensor_2}$ activate, $\prob{\sensor \text{ activates} | \event} = x$ and 
$\prob{\exists i \in \set{1,2}: \sensor_i \text{ activates} | \event} = 1 -
(1-x_1)(1-x_2) = x - x_1x_2 < x$.
}
\end{proof}

\section{Results in the Star Network\\ Model}  \label{sec:lazy-appendix}

In this section we will prove that lazy renormalization samples
sensors from a proper scaled distribution $(1 - e^{-\alpha})p_{\sensor}$ where $p_{\sensor}$ is the input distribution.
We then 
bound the communication
overhead of using lazy renormalization for any MAB algorithm
satisfying certain assumptions enumerated below, and then show how
these bounds apply to \EXPthree.

\begin{prop}
The lazy renormalization scheme of~\secref{ssec:lazy}, described in pseudocode
in~\figref{fig:lazy-renorm-code}, samples $\sensor$ with probability 
$(1 - e^{-\alpha})p_{\sensor}$,  where $p_{\sensor} =
\rho(w_{\sensor}(t), Z(t))$ is the desired probability mass for $\sensor$.
\end{prop}

\begin{proof}
Lazy renormalization
selects each sensor $\sensor$ with probability
$(1-e^{-\alpha})p_{\sensor}$, because of the way the random bits
$r_{\sensor}$ are shared in order to implement a coupled distribution for
sensor activation and selection.
Note that it would be sufficient to run the \LimitProtocol on the correct (possibly
oversampled) probabilities, $\alpha p_{\sensor}$, 
since then~\propref{thm:sampling} ensures that each $\sensor$
is selected with probability $(1 - e^{-\alpha}) p_{\sensor}$.
The difficulty is that $\sensor$ does not have access to the correct
normalizer $Z(t)$, but only its estimate (lower bound) for it,
$Z_{\sensor}(t)$.
To overcome this difficulty, we define a joint probability
distribution over two random variables $(X_{\sensor},
Y_{\sensor})$, where 
\[
X_{\sensor}  = X_{\sensor}(R) := \left\{ 
\begin{array}{ll} 
1 & \text{if $R \ge 1 - \alpha \cdot\rho(w_{\sensor}(t), Z_{\sensor}(t))$} \\
0 & \text{otherwise}
\end{array}  \right.
\]
\[
Y_{\sensor} = Y_{\sensor}(R) := \min \set{b \ :\ \sum_{a=0}^b \frac{e^{-\lambda}
    \lambda^{a}}{a!} \ge R }
\]
and $\lambda := \alpha \cdot \rho(w_{\sensor}(t), Z(t))$, and $R$ is
sampled uniformly at random from $[0,1]$.
Now, note that $Y_{\sensor}$ is distributed as $\Poisson{\lambda}$.
Also note that $Y_{\sensor} \ge 1$ implies $X_{\sensor} \ge 1$,
because $Y_{\sensor} \ge 1$ implies $R \ge e^{-\lambda}$ and 
\[
 e^{-\lambda}  \ge 1 - \lambda \ge 1 - \alpha \cdot \rho(w_{\sensor}(t), Z_{\sensor}(t))
\]
since $1 + x \le e^x$ for all $x \in \reals$, and $\rho(w_{\sensor}(t),
Z_{\sensor}(t)) \ge \rho(w_{\sensor}(t), Z(t))$ due to fact
that $Z_{\sensor}(t) \le Z(t)$.
It follows that we can use the event $X_{\sensor} \ge 1$ as a
conservative indicator that $\sensor$ should activate.
In this case, it will send its sampled value for $R$, namely
$r_{\sensor}$, and its weight $w_{\sensor}(t)$ to the server.  
The server knows $Z(t)$, and then can
use $r_{\sensor}$ and $w_{\sensor}(t)$ to compute
$Y_{\sensor}(r_{\sensor})$, 
the sample from $\Poisson{\lambda}$ that $\sensor$ would have
drawn had it known $Z(t)$.  The resulting distribution on selected
sensors is thus exactly the same as in the \LimitProtocol without lazy
renormalization.
Invoking~\propref{thm:sampling} thus completes the proof.
\end{proof}

We now describe the assumptions that are sufficient to ensure lazy
renormalization has low communication costs.
Fix an action $\sensor$ and a multiarmed bandit algorithm.
Let $p_{\sensor}(t) \in [0,1]$ be the random variable denoting the
probability the algorithm assigns to $\sensor$ on round $t$.
The value of $p_{\sensor}(t)$ depends on the random choices made by
the algorithm and the payoffs observed by it on previous rounds.
We assume the following about each $p_{\sensor}(t)$.
\begin{enumerate}
\item $p_{\sensor}(t)$ can be computed from local information
  $\sensor$ possesses and global information the server has.
\item There exists an $\epsilon > 0$ such that $p_{\sensor}(t)
  \ge \epsilon$ for all $t$.
\item $p_{\sensor}(t) < p_{\sensor}(t+1)$ implies $\sensor$ was
  selected in round $t$.
\item There exists $\hat{\epsilon} > 0$ such that 
  $p_{\sensor}(t+1) \ge p_{\sensor}(t)/(1+\hat{\epsilon})$ for all $t$.
\end{enumerate}
Many MAB algorithms satisfy these conditions.
For example, all MAB algorithms with non-trivial no-regret guarantees
against adversarial payoff functions must continually explore all
their options, which effectively mandates $p_{\sensor}(t) \ge
\epsilon$ for some $\epsilon > 0$.  In Lemma~\ref{lem:exp3conditions}
we prove that \EXPthree does so with $\epsilon = \gamma/\numsensors$ and
$\hat{\epsilon} = (e-1)\frac{\gamma}{\numsensors}$, assuming payoffs in $[0,1]$.
In this case, Theorem~\ref{thm:lazy_renorm_ratio} bounds the expected
increase in sensor communications due to lazy renormalization by a
factor of $1  + \frac{e-1}{\alpha}$.

\begin{theorem} \label{thm:lazy_renorm_ratio}
Fix a multiarmed bandit instance with possibly adversarial payoff
functions, and a MAB algorithm 
satisfying the above assumptions on its distribution over actions
$\set{p_{\sensor}(t)}_{\sensor \in \sensors}$.
Let $q_{\sensor}(t)$ be the corresponding random estimates for
$p_{\sensor}(t)$ maintained under lazy renormalization with
oversampling parameter $\alpha$.
Then for all $\sensor$ and $t$, 
\[
\E{q_{\sensor}(t)/p_{\sensor}(t)} \le 1 + \frac{\hat{\epsilon}}{\alpha\epsilon}
\]
and 
\[
\E{q_{\sensor}(t)} \le \paren{1 +
  \frac{\hat{\epsilon}}{\alpha\epsilon}} \E{p_{\sensor}(t)}
\mbox{.}
\]
\end{theorem}

\begin{proof}
Fix $\sensor$, and let $p(t) := p_{\sensor}(t)$, $q(t) := q_{\sensor}(t)$.
We begin by bounding $\prob{q(t) \ge \lambda p(t)}$ for $\lambda \ge
1$.  Let $t_0$ be the most recent round in which $q(t_0) = p(t_0)$.
We assume $q(0) = p(0)$, so $t_0$ exists.
Then $q(t) = p(t_0) \ge \lambda p(t)$  implies $p(t_0)/p(t)
\ge \lambda$.  By assumption $p(t')/p(t'+1) \le (1+\hat{\epsilon})$
for all $t'$, so $p(t_0)/p(t) \le (1+\hat{\epsilon})^{t-t_0}$.
Thus $\lambda \le (1+\hat{\epsilon})^{t-t_0}$ and 
$t-t_0 \ge \ln(\lambda)/ \ln (1+\hat{\epsilon})$. 
Define $t(\lambda) := \ln(\lambda)/ \ln (1+\hat{\epsilon})$.

By definition of $t_0$, there were no activations under lazy renormalization in rounds $t_0$
through $t-1$ inclusive, which occurs with probability 
$\prod_{t' = t_0}^{t-1} (1-\alpha q(t'))$ $=$ $(1-\alpha
q(t))^{t-t_{0}} $ $\le$ $(1-\alpha q(t))^{\ceil{t(\lambda)}}$, where $\alpha$ is the
oversampling parameter in the protocol.  
We now bound $\Econd{q(t)/p(t)}{q(t)}$.
Recall that $\E{X} = \int_{x = 0}^{\infty} \prob{X \ge x}dx$ for any non-negative
random variable $X$.  It will also be convenient to define 
$\lamexp := \ln(1/(1 - \alpha q(t))) / \ln(1+\hat{\epsilon})$ and assume for now
that $\lamexp > 1$.
Conditioning on $q(t)$, we see that 
\[
\begin{array}{lcl}
 \vspace{2mm}
\Econd{q(t)/p(t)}{q(t)} & = & \int_{\lambda = 0}^{\infty} \prob{q(t) \ge
  \lambda p(t)}d\lambda \\ 
 \vspace{2mm}
  & = & 1 + \int_{\lambda = 1}^{\infty} \prob{q(t) \ge
  \lambda p(t)}d\lambda \\ 
\vspace{2mm}
 & \le & 1 + \int_{\lambda = 1}^{\infty} (1-\alpha q(t))^{t(\lambda)}d\lambda\\
 \vspace{2mm}
 & = & 1 + \int_{\lambda = 1}^{\infty}
\lambda^{\ln(1-\alpha q(t))/\ln(1+\hat{\epsilon}) }d\lambda \\
 \vspace{2mm}
 & = & 1 + \int_{\lambda = 1}^{\infty} \lambda^{-\lamexp}d\lambda \\ 
 \vspace{2mm}
 & = & 1 + \frac{1}{\lamexp - 1}
\end{array}
\]
Using $\ln\paren{\frac{1}{1-x}} \ge x$ for all $x < 1$ and 
$\ln(1+x) \le x$ for all $x > -1$, we can show that 
$\lamexp \ge \alpha q(t)/\hat{\epsilon}$
so $1+\frac{1}{\lamexp - 1} \le \alpha
q(t)/(\alpha q(t) - \hat{\epsilon})$.
Thus, if $\alpha q(t) > \hat{\epsilon}$ then $\lamexp > 1$ and we
obtain 
$\Econd{q(t)/p(t)}{q(t)} \le  \alpha q(t)/(\alpha q(t) - \hat{\epsilon})$.

If $q(t) >> \hat{\epsilon}$, this gives a good bound.  If $q(t)$ is small, we
rely on the assumption that $p(t) \ge \epsilon$ for all $t$ to get a trivial
bound of $q(t)/p(t) \le q(t)/\epsilon$.
We thus conclude 
\begin{equation} \label{eqn:ratio-bound}
\Econd{q(t)/p(t)}{q(t)} \le \min \paren{\alpha q(t)/(\alpha q(t) - \hat{\epsilon}) ,\ q(t)/\epsilon}\mbox{.}
\end{equation}
Setting $q(t) = (\hat{\epsilon}/\alpha + \epsilon)$ to maximize this
quantity yields an unconditional bound of 
$\E{q(t)/p(t)} \le 1 + \hat{\epsilon}/\alpha\epsilon$.

To bound $\E{q(t)}$ in terms of $\E{p(t)}$, note that for all $q$
\begin{align*}
q/\Econd{p(t)}{q(t)=q} & \le \Econd{q(t)/p(t)}{q(t)=q}\\
                    & \le 1 + \hat{\epsilon}/\alpha\epsilon
\end{align*}
where the first line is by Jensen's inequality, and the second is by 
equation~\ref{eqn:ratio-bound}.
Thus $q \le \paren{1 +
  \hat{\epsilon}/\alpha\epsilon}\Econd{p(t)}{q(t)=q}$ for all $q$.
Taking the expectation with respect to $q$ then proves 
$\E{q_{\sensor}(t)} \le \paren{1 +
  \frac{\hat{\epsilon}}{\alpha\epsilon}} \E{p_{\sensor}(t)}$
as claimed.
\end{proof}

\begin{lemma} \label{lem:exp3conditions}
\EXPthree with $\eta = \gamma/n$ 
satisfies the conditions of Theorem~\ref{thm:lazy_renorm_ratio} 
with $\epsilon = \gamma/\numsensors$ and $\hat{\epsilon} =
(e-1)\frac{\gamma}{\numsensors}$.
\end{lemma}

\begin{proof}
The former equality is an easy observation.
To prove the latter equality, fix a round $t$ and a selected action
$\sensor$.  Let 
$w_{\sensor}(t)$ be the weight of $\sensor$ in round $t$, and $W(t)$ be the
total weight of all actions in round $t$.
Let $\payoff$ be the payoff to $\sensor$ in round $t$.
Given the update rule 
$w_{\sensor}(t+1) =
w_{\sensor}(t)\exp\paren{\frac{\gamma}{\numsensors}\frac{\payoff(\sensor,t)}{p_{\sensor}(t)}}$,
only the probabilities of the other actions will be decreased.
It is not hard to see that they will be decreased by a multiplicative factor of at most 
$W(t)/W(t+1)$, no matter what the learning parameter $\gamma$ is.
By the update rule, 
\[
W(t+1) = W(t) +
w_{\sensor}(t)\paren{\exp\paren{\frac{\gamma}{\numsensors}
\frac{\payoff}{p_{\sensor}(t)}}
    - 1}\mbox{.}
\]
Let $p := p_{\sensor}(t)$ and $x :=
\frac{\gamma}{\numsensors}\payoff$.  
Dividing the above equation by $W(t)$, we get 
\begin{align}
\frac{W(t+1)}{W(t)} & = 1 + p\paren{\exp\paren{x/p} - 1} \\
 & \le 1 + p\paren{x/p + (e-2)(x/p)^2}\\
 & \le 1 + x + (e-2)x^2/p
\end{align}
where in the second line we have used 
$e^x \le 1 + x + (e-2)x^2$ for $x \in [0,1]$.
Note $\payoff \le 1$ implies $x \le \gamma/n \le p$, so 
$\frac{W(t+1)}{W(t)} \le 1 + (e-1)x \le 1 +
(e-1)\frac{\gamma}{\numsensors}$.
It follows that setting $\hat{\epsilon} =
(e-1)\frac{\gamma}{\numsensors}$ is sufficient to ensure 
$p_{\sensor}(t+1) \ge p_{\sensor}(t)/(1+\hat{\epsilon})$ for all $t$.
\end{proof}

We now prove Theorem~\ref{thm:dexp3-performance} and Corollary~\ref{cor:dexp3-performance}.

\begin{proof}[of Theorem~~\ref{thm:dexp3-performance}.]
We prove in Lemma~\ref{lem:exp3conditions} that \EXPthree satisfies
the conditions of Theorem~\ref{thm:lazy_renorm_ratio} 
with $\epsilon = \gamma/\numsensors$ and $\hat{\epsilon} =
(e-1)\frac{\gamma}{\numsensors}$.  Thus by Theorem~\ref{thm:lazy_renorm_ratio} 
\begin{align*}
\E{\sum_{\sensor} q_{\sensor}(t)} & \le \paren{1 +
  (e-1)/\alpha}\E{\sum_{\sensor} p_{\sensor}(t)}\\ & = \paren{1 +
  (e-1)/\alpha} 
\end{align*}
because $\sum_{\sensor} p_{\sensor}(t) = 1$.
Each sensor $\sensor$ activates with probability $\alpha q_{\sensor}(t)$,
so the expected number of activations is 
\[
\E{\alpha \sum_{\sensor} q_{\sensor}(t)} \le \alpha\paren{1 +
  (e-1)/\alpha}\mbox{.}
\]
That proves the claimed bounds in expectation.  To prove bounds with
high probability, note that a sensor activates with probability
$\alpha q_{\sensor}(t)$ in round $t$, where   $q_{\sensor}(t)$ is a
random variable.
Fix $t$.  Let $[\event]$ denote the indicator variable
for the event $\event$, i.e., $[\event] = 1$ if $\event$ occurs, and
$[\event] = 0$ otherwise.  Then we can write 
$[\sensor \text{ activates in round }t] = [\alpha
q_{\sensor}(t) \ge R]$, where $R$ is sampled uniformly at random from
$[0,1]$ and $R$ is independent of $q_{\sensor}(t)$.
Then if $f_R$ is the probability density functions of $R$ we can write 
\[
\begin{aligned}
   \prob{R \le \alpha q_{\sensor}(t)} & = \int_{r = 0}^1 \prob{\alpha
q_{\sensor}(t) \ge R \ |\  R = r} f_R(r) dr\\
  & = \int_{r = 0}^1 \prob{\alpha q_{\sensor}(t) \ge r} f_R(r) dr \\
& = \int_{r = 0}^1 \prob{\alpha q_{\sensor}(t) \ge r}dr \\
& = \E{ \alpha q_{\sensor}(t) }
\end{aligned}
\]
Thus the number of sensor activations is a sum of $|\sensors|$ binary
random variables with cumulative mean 
$\mu := \sum_{\sensor}  \E{ \alpha q_{\sensor}(t) }$.
We have already bounded this mean as $\mu \le \alpha + (e-1)$.
From here a simple application of a Chernoff-Hoeffding bound suffices
to prove that with high probability this sum is at most 
$\cO(\alpha + \log n)$.  Let $A$ be the number of sensor activations.
Then, e.g., Theorem~$5$ of~\cite{CL06} 
immediately implies 
\[
\prob{A \ge \mu(1 + \delta)} \le \exp \paren{ - \frac{\delta^2
    \mu^2}{2 \mu + \frac{2\delta \mu}{3}}}
\]
For $\delta \ge 1$, this yields 
$\prob{A \ge \mu(1 + \delta)} \le \exp \paren{ -
  \frac{3\delta\mu}{8}}$.
Setting $\delta = 1 + \frac{8c\ln n}{3\mu}$ ensures this probability is
at most $n^{-c}$, hence 
$\prob{A \ge 2\mu + \frac{8}{3}c\ln n} \le n^{-c}$.
Noting that $\mu \le \alpha + (e-1)$ completes the high probability
bound on the number of activations.

As for the number of messages, note that each message involves a
sensor as sender or receiver, and by inspection the protocol only
involves two messages per activated node.
\end{proof}

\begin{proof}[of Corollary~\ref{cor:dexp3-performance}.]
Use the distributed \EXPthree protocol with lazy
renormalization with $\alpha = \ln n$.  We have already established
that the probability of nothing being selected is $e^{-\alpha}$ or
$1/n$ in this case.  If nothing is selected, send out $n$ messages,
one to each sensor, to rerun the protocol.  The expected number of
messages sent to initiate additional runs of the protocol is 
$\sum_{x=1}^{\infty} nx/n^{x} = \paren{1 - 1/n}^{-2} = 1 + \cO(1/n)$.
Let $X$ be the number of sensor activations.
As in the proof of
Proposition~\ref{prop:broadcast-activations}, if $Y$ is 
the expected number of sensor activations without rerunning the
protocol when nothing is selected, then 
$\E{X} = \E{Y}/\prob{Y \ge 1}$.  By
Theorem~\ref{thm:dexp3-performance} $\E{Y} \le  \alpha\paren{1 +
  (e-1)/\alpha}$.  Since $\prob{Y \ge 1} = 1 - e^{-\alpha}$, we 
conclude 
\[
\E{X} \le \ln n + (e-1) + \cO\paren{\frac{\ln n}{n}}\mbox{.}
\]
The with-high-probability bounds on the number of sensor activations
are proved as in 
the proof of Corollary~\ref{cor:dexp3-performance}.

As for the number of messages, note that other than 
messages sent to initiate additional runs of the protocol, there are
only two messages per activated node.  Finally, the regret bounds for
distributed \EXPthree are the same as standard \EXPthree because by
design the two algorithms select sensors from exactly the same distribution in each
round.  Note that the distribution in any given round is a random
object depending on the algorithm's choices in the previous rounds,
however on each round the distribution on distributions is the same for both \EXPthree
variants, as can be readily proved by induction on the round number. 
\end{proof}

\section{Algorithm \OGunit with Faulty Actions} \label{sec:faulty}

In order to prove Theorem~\ref{thm:better-lazydog}, we need a
guarantee on the performance of \OGunit if its elements are may fail
to give any benefit.  We provide this in the form of
Theorem~\ref{thm:faulty-sensors}.

Suppose we run \dog with the \LimitProtocol with lazy renormalization,
and do not resample on stages where no sensor activates.
Then with some probability during any given stage $i \in [\K]$, no
sensors activate and the server receives no information.
Suppose that this probability is at most $\delta$ in each stage.
We have shown in section~\ref{sec:sampling} that $\delta \le e^{-\alpha}$ where $\alpha$ is the
oversampling parameter.  We claim that we can compensate for this
possibility by running \dog for $\K/(1-\delta)$ stages in each round
rather than $\K$, because of the following guarantee for \OGunit.

\begin{theorem} \label{thm:faulty-sensors}
Fix finite set $\sensors$, $\K \in \nats$, and a sequence of 
monotone submodular functions $f_1, \ldots, f_T:2^{\sensors} \to [0,1]$.
Let $\OPT_{\K} = \max_{S \subset \sensors, |S| \le \K} \sum_{t=1}^{t} f_t(S)$.
For all $\sensor \in \sensors$ let $\sensor'$ be a random element which is 
$v$ with probability $1-\delta_\sensor$ and is $\NULL$\footnote{Here,
  a $\NULL$ element always contributes nothing in the way of utility,
  so that $f_t(S \cup \set{\NULL}) = f_t(S)$ for all $t$ and $S$.}
 with probability
$\delta_\sensor$.  
Let $f'_t(S') := \E{f_t(S)}$ where $S$ the set obtained
by including every element $\sensor'$ of $S'$ in it independently with
probability $\delta_{\sensor}$.
Let $S'_1, \ldots, S'_T$ be the sequence of random sets obtained from
running \OGunit with actions $\sensors' := \set{\sensor' : \sensor \in
  V}$ and objective functions $\set{f'_t}_{t = 1}^T$ and $\K' =
\K/(1-\delta)$ stages, where $\delta = \max_{\sensor}
\delta_{\sensor}$.
Suppose the algorithms for each stage have expected regret at most $r$.
Then 
\[
\E{\sum_{t=1}^T f'_t(S'_t)} \ge \paren{1-\frac{1}{e}}\OPT_{\K} - \ceil{\frac{\K}{1-\delta}} r\mbox{.}
\]
\end{theorem}
 
\begin{proof}
It suffices to prove the analogous result in the offline case; the
``meta-actions'' analysis in~\cite{streeter07tr} can then be used to complete the proof.
So consider a set of elements $\sensors$ and the ``faulty'' versions
$\sensors'$.  Fix a monotone submodular $f:2^{\sensors} \to [0,1]$ and
define $f'$ as above.  Run the offline greedy algorithm on $f'$ to try
to find the best set of $\K' = \frac{\K}{1-\delta}$ elements in $\sensors'$.
Let $g'_i$ be the chosen element in stage $i$, and let $G'_i = \set{g'_j : 1
  \le j \le i}$.  Let $G_i$ denote the realization of $G'_i$ after
sampling, so that $G_i \subseteq \set{g : g' \in G'_i}$.
Let  $S^* = \argmax_{S \subseteq \sensors, |S| \le \K}(f(S))$.
We claim that for all $i$
\[
\Econd{f(G'_{i+1}) - f(G'_i)}{G_i} \ge \paren{1-\delta}
\frac{f(S^*) - f(G_i)}{\K}
\]
because 
\begin{align*}
f(S^*) - f(G_i) & \le f(G_i \cup S^*) - f(G_i)\\
                & \le \sum_{\sensor \in S^*} \paren{f(G_i + \sensor) - f(G_i)}\\
                & \le \K \cdot \max_{\sensor} \paren{f(G_i + \sensor) - f(G_i)}
\end{align*}
and $\max_{\sensor'} \paren{\Econd{f(G'_i + \sensor') - f(G'_i)}{G_i} }$ is at least equal to 
$(1-\delta)\max_{\sensor} \paren{f(G_i + \sensor) - f(G_i)}$.
Removing the conditioning on $G_i$ we get 
\[
\E{f(G'_{i+1}) - f(G'_i)} \ge \paren{1-\delta}\frac{f(S^*) -
  \E{f(G'_i)}}{\K}
\]
Let $\Phi(i) = f(S^*) - \E{f(G'_{i})}$.  The previous equation implies 
$\Phi(i+1) \le \Phi(i)\paren{1 - \frac{1-\delta}{\K}}$.
By induction $\Phi(i) \le f(S^*)\paren{1 - \frac{1-\delta}{\K}}^i$.
Using $1-x \le e^{-x}$ we conclude that $\Phi(\ceil{\K/(1-\delta)})
\le f(S^*)/e$ and $f'(G_{\K'}) \ge \paren{1-\frac{1}{e}}f(S^*)$.
\end{proof}

\noindent
We are now ready to prove Theorem~\ref{thm:better-lazydog}.

\begin{proof}[Theorem~\ref{thm:better-lazydog}]
To bound the number of sensor activations, we note there are 
 $\K' := \ceil{\K/(1 - e^{-\alpha})}$ rounds, and each round activates  
at most $\alpha + (e-1)$ sensors in expectation and $\cO(\alpha + \log
n)$ sensors with high probability by
Theorem~\ref{thm:dexp3-performance} (which proves these bounds in the
higher communication case where we do rerun the \LimitProtocolShort protocol if nothing is
selected).  This, and the fact that $\alpha/(1 - e^{-\alpha}) =
O(\alpha)$ for $\alpha > 0$ yields the claimed activation bounds.
It is an easy observation that the number of messages is at most twice
the number of activations.  Clearly, at most one sensor per stage is
activated, so at most $\K'$ are activated over one round.
Finally, the regret bound follows from
Theorem~\ref{thm:faulty-sensors}, using $\delta = e^{-\alpha}$.
\end{proof}

}{ }